\documentclass[10pt, a4paper]{article}

\usepackage[utf8]{inputenc}
\usepackage{amsmath,amssymb,amsthm}
\usepackage{graphicx}
\usepackage{url}
\usepackage{algorithm, algorithmic}
\usepackage{booktabs}
\usepackage{multirow}
\usepackage{tabularx}
\usepackage{xcolor}
\usepackage{rotating}
\usepackage{adjustbox}
\usepackage{epstopdf}
\usepackage{float}
\usepackage{makecell}
\usepackage{colortbl}

\usepackage[top=2.5cm, bottom=2.5cm, left=2.5cm, right=2.5cm]{geometry}

\usepackage{fancyhdr}
\usepackage{titlesec}
\titleformat{\section}{\large\bfseries}{\thesection}{1em}{}
\titleformat{\subsection}{\bfseries}{\thesubsection}{1em}{}
\titleformat{\subsubsection}{\bfseries}{\thesubsubsection}{1em}{}

\usepackage{authblk}

\makeatletter
\newenvironment{breakablealgorithm}
{
		\begin{center}
			\refstepcounter{algorithm}
			\hrule height.8pt depth0pt \kern2pt
			\renewcommand{\caption}[2][\relax]{
				{\raggedright\textbf{\ALG@name~\thealgorithm} ##2\par}%
				\ifx\relax##1\relax 
				\addcontentsline{loa}{algorithm}{\protect\numberline{\thealgorithm}##2}%
				\else 
				\addcontentsline{loa}{algorithm}{\protect\numberline{\thealgorithm}##1}%
				\fi
				\kern2pt\hrule\kern2pt
			}
		}{
		\kern2pt\hrule\relax
	\end{center}
}

\newtheorem{theorem}{Theorem}[section]
\numberwithin{equation}{section}
\makeatother

\AtBeginDocument{%
  }

\begin{document}

\title{A Multi-Level Framework for Multi-Objective Hypergraph Partitioning: Combining Minimum Spanning Tree and Proximal Gradient}

\author[1]{Yingying Li}
\author[1]{Mingxuan Xie}
\author[1]{Hailong You}
\author[2]{Yongqiang Yao}
\author[1,*]{Hongwei Liu}

\affil[1]{Xidian University, Xi'an, Shaanxi, China}
\affil[2]{Shihezi University, Shihezi, Xinjiang, China}

\affil[*]{Corresponding author: hwliuxidian@163.com}

\date{} 

\maketitle

\begin{center}
    \small
    \textbf{Emails:} 
    Yingying Li: yingyl99@163.com,
    Mingxuan Xie: xmx\_0417@163.com,
    Hailong You: hlyou@mail.xidian.edu.cn,
    Yongqiang Yao: yqyao2021@163.com,
    Hongwei Liu: hwliuxidian@163.com \\
\end{center}

\begin{abstract}
  This paper proposes an efficient hypergraph partitioning framework based on a novel multi-objective non-convex constrained relaxation model. A modified accelerated proximal gradient algorithm is employed to generate diverse $k$-dimensional vertex features to avoid local optima and enhance partition quality. Two MST-based strategies are designed for different data scales: for small-scale data, the Prim algorithm constructs a minimum spanning tree followed by pruning and clustering; for large-scale data, a subset of representative nodes is selected to build a smaller MST, while the remaining nodes are assigned accordingly to reduce complexity. To further improve partitioning results, refinement strategies including greedy migration, swapping, and recursive MST-based clustering are introduced for partitions. 
  
  Experimental results on public benchmark sets demonstrate that the proposed algorithm achieves reductions in cut size of approximately 2\%--5\% on average compared to KaHyPar in 2, 3, and 4-way partitioning, with improvements of up to 35\% on specific instances. Particularly on weighted vertex sets, our algorithm outperforms state-of-the-art partitioners including KaHyPar, hMetis, Mt-KaHyPar, and K-SpecPart, highlighting its superior partitioning quality and competitiveness. Furthermore, the proposed refinement strategy improves hMetis partitions by up to 16\%. A comprehensive evaluation based on virtual instance methodology and parameter sensitivity analysis validates the algorithm's competitiveness and characterizes its performance trade-offs.
\end{abstract}


\section{Introduction}\label{sec1}
Hypergraphs are a generalization of graphs where each hyperadge can connect more then two vertices.
A weighted undirected hypergraph \(\mathcal{H} = (\mathcal{V}, \mathcal{E}, B, \varpi)\) consists of a set of \(n\) vertices \(\mathcal{V} = \{v_1, v_2, \dots, v_n\}\), a set of \(m\) hyperedges \(\mathcal{E} = \{e_1, e_2, \dots, e_m\}\), a weight function of the vertex \(B: \mathcal{V} \to \mathbb{R}\), and a net weight function \(\varpi: \mathcal{E} \to \mathbb{R}\).

In practical applications, the hypergraph partitioning problem is widely used to optimize circuit layout performance \cite{VLSI1999}, data center network design \cite{CenterNetwork2015}, distributed database storage \cite{Social2017}, and transaction cost optimization in distributed systems \cite{Community2015}. Specifically, it aims to divide the vertex set $\mathcal{V}$ into $k$ mutually disjoint subsets ${\Pi_1, \Pi_2, \dots, \Pi_k}$, such that the cut cost is minimized while satisfying load balance constraints. The mathematical formulation is as follows:
\begin{align}
\text{cutsize} = \sum_{e \in \mathcal{E}} \varpi_e (\mu_e - 1),
\end{align}
where $\mu_{e}$ is the number of partitions spanned by the hyperedge $E$.
Simultaneously satisfying the constraint conditions,
\begin{itemize}
\item[\textbf{a:}] $\bigcup_{i=1}^{k} \Pi_i = \mathcal{V}$, and $\Pi_i \cap \Pi_j = \emptyset$ for all $i \neq j$;
  \item[\textbf{b:}] $\sum_{v \in \Pi_i} B(v) \leq U_i := (1 + \varepsilon) \left\lceil \frac{\sum_{v \in \mathcal{V}} B(v)}{k} \right\rceil$,
\end{itemize}
where $\varepsilon$ is a balance factor, $U_i$ is the upper bound constraint of resources.

\subsection{Related works}

As early as 1970, Kernighan and Lin \cite{KL1970} proposed the classical graph bisection algorithm (KL algorithm), which begins with an initial partition obtained through breadth-first traversal and iteratively swaps nodes to reduce the number of cut edges. Later, Fiduccia and Mattheyses \cite{FM1982} extended this approach to hypergraphs and introduced the FM algorithm. This method evaluates each vertex by its gain the estimated reduction in cut cost when moved and greedily selects the vertex with the highest gain in each iteration. It then updates the gains of neighboring vertices and ultimately applies the sequence of moves yielding the highest cumulative gain. Building upon this foundation, numerous studies have enhanced and extended the FM algorithm for more complex scenarios. For instance, Dasdan et al. \cite{dasdan1997two} proposed an efficient bidirectional FM algorithm, while Dutt et al. \cite{dutt2002cluster} refined the cluster-based partitioning strategy. Others have extended FM to $k$-way partitioning \cite{kwayfirst, Tang2024}, improving gain computation and node assignment mechanisms.

In multilevel partitioning frameworks, FM is widely adopted in the refinement stage to locally adjust vertex placement and improve partition quality. These frameworks typically consist of three stages: coarsening, initial partitioning, and refinement. Representative tools include hMetis \cite{khmetis1999} and KaHyPar \cite{kahypar2020advanced, kahypar2022parallel, kahypar2023high} and so on.

In recent years, with the widespread application of hypergraphs in circuit design, parallel computing, and data management, the demand for efficient and high quality hypergraph partitioning methods has significantly increased. Traditional multilevel partitioning frameworks are known for their high computational efficiency, especially on large-scale datasets. However, they often fall short in solution quality particularly in complex scenarios involving heterogeneous constraints or global optimization objectives due to their susceptibility to local optima. To address these limitations, researchers have increasingly explored optimization driven paradigms, such as integer linear programming (ILP), structural modeling, and continuous optimization, to formulate more expressive models for hypergraph partitioning and achieve better solutions.

In 2024, Bustany et al.~\cite{KSpecPart} proposed K-SpecPart, which first constructed a path graph through vertex embedding and generated a minimum spanning tree (MST) or a low-stretch spanning tree (LSST). Dynamic programming was then employed to reassign edge weights, better reflecting the original hypergraph’s cut structure. Based on this weighted tree, the method performed multi-way partitioning using recursive bisection, METIS tools, or VILE heuristics. Furthermore, ILP was applied to integrate multiple candidate solutions and enhance the final partition quality. Experimental results showed that the method achieved up to a 50\% reduction in cut cost while maintaining strong scalability on large-scale datasets, significantly improving both efficiency and partitioning quality. More clustering methods based on minimum spanning tree can be found in the literature
\cite{KHAN20221113, RDMN2022, MISHRA2020}.

Building upon the structural modeling paradigm, TritonPart \cite{Tritionpart} integrated a multilevel framework to enhance both modeling expressiveness and solution quality. In its coarsening phase, multiple candidate coarse hypergraphs were generated via various vertex ordering strategies and edge aggregation heuristics. For initial partitioning, TritonPart leveraged ILP solvers (e.g., OR-Tools as an open-source alternative to CPLEX), guided by spectral embeddings provided by SpecPart. In the refinement stage, FM (Fiduccia–Mattheyses) or greedy hyperedge refinement (HER) algorithms were employed to improve partition quality locally. To further enhance global structure, TritonPart adopted a cut-overlay clustering strategy to fuse multiple candidate partitions into a new reduced hypergraph, enabling recursive refinement under a V-cycle framework. While TritonPart offered strong support for practical constraints such as fixed nodes, grouping, and balance control, our focus in this work was on evaluating general-purpose multilevel partitioners; hence, we limited our empirical comparisons to well-established baselines such as KaHyPar and hMetis.

It is noteworthy that although ILP models theoretically provide globally optimal solutions for hypergraph partitioning problems, their NP-hard nature often leads to prohibitively long computation times and excessive memory usage in large-scale practical applications. Mainstream commercial solvers such as CPLEX and Gurobi struggle to efficiently handle instances with rapidly growing numbers of variables and constraints. To overcome these difficulties, researchers have widely adopted continuous optimization techniques as alternatives, relaxing the original ILP formulations into LP, semidefinite programming (SDP), or unconstrained optimization problems, and employing numerical methods such as gradient descent, Lagrangian multipliers, and conjugate gradient methods for efficient solving. For example, Nguyen et al. \cite{Nguyen2016LP} transformed graph partitioning ILPs with knapsack constraints into compact LP models and combined heuristic strategies to construct upper bounds; Wu et al. \cite{WU2019191} proposed a continuous relaxation model based on Deterministic annealing neural networks (DANN), introducing an entropy potential function and achieving global convergence via Lagrangian methods; Wiegele and Zhao \cite{Wiegele2022COA} leveraged SDP relaxations and alternating direction method of multipliers (ADMM) to obtain tighter lower bounds; Sun et al. \cite{sun2025} further reformulated weighted and fixed-constraint partitioning problems as unconstrained optimizations and proposed an efficient recursive bipartitioning algorithm based on subspace minimization and conjugate gradient methods. These approaches significantly enhance solution quality and computational efficiency in large-scale partitioning, demonstrating the broad applicability and theoretical depth of continuous optimization in this domain.

Despite the widespread use of continuous optimization methods in graph and hypergraph partitioning, some classical optimization algorithms remain underexplored. The proximal gradient method is one such promising approach, recognized for its simple iterative structure and efficient computation, with successful applications in non-smooth optimization problems in image processing and machine learning. Its standard iteration is expressed as
\[
x_{k+1} = \mathrm{prox}_{\alpha_k g}(x_k - \alpha_k \nabla F(x_k)),
\]
where $\alpha_k$ is the step size, and $\mathrm{prox}_{\alpha_k g}$ denotes the proximal operator of $\alpha_k g$. In 2024, Wang et al. \cite{wang2024class} proposed a modified accelerated proximal gradient for nonconvex case (\textbf{modAPG\_nc}) and rigorously established local convergence rates under the Kurdyka–\L{}ojasiewicz (K\L{}) property of the objective function. Although this method has demonstrated excellent performance in other optimization scenarios, its extension to non-convex, non-smooth hypergraph partitioning remains unexplored, presenting new research opportunities and potential for developing efficient partitioning algorithms based on \textbf{modAPG\_nc}.

\subsection{Contributions}
Inspired by the above works, we integrate multiple strategies to propose a novel hypergraph partitioning approach.
The principal contributions of this paper are as follows.
\begin{enumerate}
\item

This paper follows a multilevel partitioning framework. In the initial partitioning phase, the coarse hypergraph is transformed into a graph and modeled as a constrained multi-objective optimization problem. A modified accelerated proximal gradient algorithm is employed to generate $k$-dimensional vertex features. By varying multiple parameters, different feature matrices are obtained, producing diverse partition solutions that effectively avoid local optima.
\item 
    The proposed framework includes two MST-based partitioning methods for varying data scales.
    Small-scale data is partitioned by constructing and cutting an MST via the Prim algorithm. For large-scale data, the top 20\% weighted nodes form a representative MST, and remaining nodes are assigned accordingly. Combining pruning and resource-constrained clustering achieves compact, balanced partitions with reduced computational complexity.
\item  
    This paper also presents optimization strategies for existing partitions. The algorithm combines the MST with a clustering method, intelligently selects two cluster centers for bipartitioning, and iteratively improves the initial partition results.
\item
    Numerical experiments on multiple benchmarks demonstrate that the proposed algorithm achieves the lowest cut sizes for 2, 3, and 4-way partitioning. It consistently outperforms state-of-the-art partitioners, with an average improvement of 2\%--5\% over KaHyPar and up to 72\% on specific instances. The algorithm produces optimal or near-optimal results in most cases, underscoring its significant superiority over existing leading methods.
 \item  The improvement strategy optimizes the hMetis partition results, reducing the cut size by up to 16\% with an overall improvement of 83\%. Parameter sensitivity analysis shows that the algorithm remains stable under perturbations.

\end{enumerate}

The remainder of this paper is organized as follows.
Section~\ref{sect2} introduces the basic definitions and overall framework.
Section~\ref{sect3} formulates the problem as a multi-objective model and solves it using the \textbf{modAPG\_nc} algorithm.
Section~\ref{sect4} presents two strategies for generating initial partitions.
Section~\ref{sect5} describes methods for refining feasible partitions and improving infeasible ones.
Section~\ref{sect6} reports experimental results that demonstrate the effectiveness of the proposed approach.
Section~\ref{sect7} concludes the paper.

\section{Preliminaries}\label{sect2}

\subsection{Basic definitions and lemmas}
A hypergraph can be transformed into a conventional graph via the \emph{clique expansion} method \cite{Karypis1997VLSI}, where vertices within each hyperedge are pairwise connected to form a clique, and edge weights are accordingly defined. Specifically, if vertices $i$ and $j$ share a set of hyperedges $\Omega(i,j)$, the corresponding adjacency matrix entry is given by:
\[
a_{ij} = \sum_{e_k \in \Omega(i,j)} \frac{\varpi_{e_k}}{|e_k| - 1},
\]
where $|e_k|$ is the number of vertices in hyperedge $e_k$, and $\varpi_{e_k}$ denotes its weight. If $i$ and $j$ are not connected, then $a_{ij} = 0$. Considering an undirected graph, we have $a_{ii} = 0$ and $a_{ij} = a_{ji}$.

The resulting adjacency matrix $A = (a_{ij}) \in \mathbb{R}^{n \times n}$ encodes the connection strength between vertex pairs. Based on $A$, the graph Laplacian matrix $L$ is defined as:
\begin{itemize}
  \item[(i)] For $i \neq j$, $L(i,j) = -a_{ij}$;
  \item[(ii)] For the diagonal entries, $L(i,i) = \sum_{j \neq i} a_{ij}$, i.e., the degree of vertex $i$.
\end{itemize}

A graph in which every pair of distinct vertices is connected by an edge is called a \emph{complete graph}. When the vertex set is partitioned into $k$ mutually disjoint subsets such that every pair of vertices belonging to different subsets is connected by an edge, and there are no edges within the same subset, the graph is called a \emph{complete $k$-partite graph}. A graph is \emph{connected} if there exists a path between any two vertices. An \emph{undirected graph} that is connected and contains no cycles is called a \emph{tree}. A tree is a special graph structure with $n$ vertices and $n-1$ edges, in which there exists exactly one path between any two vertices. For a connected weighted undirected graph with nonnegative edge weights, a \emph{minimum spanning tree} (MST) is a spanning tree of the graph whose total edge weight is minimal. Although the MST is not necessarily unique, its total weight is the smallest among all spanning trees.

\begin{theorem}[Cycle Property]\label{Cycle Property}
Let $C$ be a cycle in a connected undirected graph $G$, and let $a_0$ be the edge with the maximum weight in $C$. Then the MST of $G \setminus a_0$ is also a MST of $G$.
\end{theorem}

\begin{theorem}[Balanced Partition Maximizes Inner Product]\label{MAXBalanced}
Let a vertex set \(\mathcal{V}\) be partitioned into \(k\) disjoint subsets \(\Pi_1, \Pi_2, \dots, \Pi_k\), subject to the total sum constraint:
\[
\sum_{i=1}^k a_i = C,
\]
where \(a_i = |\Pi_i|\) in the unweighted case, or \(a_i = \sum_{v \in \Pi_i} B_v\) in the weighted case.

Then the objective function
\[
\Psi(a_1, \dots, a_k) = \sum_{1 \le i < j \le k} a_i a_j = \frac{1}{2} \left( \left( \sum_{i=1}^k a_i \right)^2 - \sum_{i=1}^k a_i^2 \right)
\]
achieves its maximum if and only if \(a_1 = a_2 = \dots = a_k = \frac{C}{k}\).

\end{theorem}

\begin{proof}
Since the total sum \(\sum_{i=1}^k a_i = C\) is fixed, maximizing \(\Psi\) is equivalent to minimizing
\[
\sum_{i=1}^k a_i^2.
\]
By the inequality of arithmetic and quadratic means, the sum of squares is minimized when all \(a_i\) are equal. Hence,
\[
a_1 = a_2 = \dots = a_k = \frac{C}{k},
\]
which leads to the maximum value of \(\Psi\).
\end{proof}

The definitions of other symbols and commands are shown in Table \ref{table11} and these symbols are applied throughout the text.
\begin{table}[!htbp]
\centering
\caption{Notations}
\label{table11}
\begin{tabular}{cl}
\hline
\textbf{Symbol} & \textbf{Description} \\
\hline
$\arg\min$ & Argument that minimizes a given function \\
$\arg\max$ & Argument that maximizes a given function \\
$\operatorname{diag}(\cdot)$  & Extracts the diagonal elements of a matrix as a vector \\
$\operatorname{Diag}(\cdot)$  & Diagonal matrix with the input vector on the diagonal \\
\textbf{1}  & $n$-dimensional column vector of all ones \\
$\lceil\;\cdot\;\rceil$ & Ceiling function (rounding up to the nearest integer) \\
$\lfloor\;\cdot\;\rfloor$ & Floor function (rounding down to the nearest integer) \\
$\cup$ & Union of sets (all elements in either set) \\
$\cap$ & Intersection of sets (elements common to both sets) \\
$\| \cdot \|_{2}$ & $\ell_2$ norm (Euclidean norm; square root of the sum of squares) \\
$|\cdot|$ & Absolute value or set cardinality \\
$\setminus$ & Set difference (elements in one set but not in another) \\
\hline
\end{tabular}
\end{table}

\subsection{The main framework of the proposed algorithm}

This work adopts a multilevel partitioning framework, where the hypergraph is coarsened via vertex clustering to reduce problem size. A relaxed continuous optimization method is used to extract structural information for guiding initial partitioning. MSTs are constructed on reduced graphs to facilitate initial cuts and local improvement. Finally, the classical $k$-way FM algorithm \cite{Tang2024} is applied during uncoarsening to refine and project the partitioning back to the original hypergraph.

During the coarsening phase, we use a matching score function to measure the strength of connection between vertex pairs. For any pair of vertices $(v_i, v_j)$, the matching score is defined as:
\begin{align}\label{eq:edge-score}
\Gamma(v_i, v_j) = \sum_{e \in I(v_i) \cap I(v_j)} \frac{\varpi_{e}}{\max(1, |e| - 1)},
\end{align}
where $I(v)$ denotes the set of hyperedges incident to vertex $v$ and $|e|$ is the number of vertices in hyperedge $e$. This function favors merging vertex pairs that frequently co-occur in smaller hyperedges, effectively preserving local structure.
Using this scoring function, we construct a matching set $M$. For each unmatched vertex $v_i$, we select a neighbor $v_j$ from its adjacent set $\mathcal{Z}(v_i)$ such that:
\begin{equation}\label{eq:matching-set}
    M = \{ (v_i, v_j) \mid v_j = \arg\max_{u \in \mathcal{V} \setminus M} \Gamma(v_i, u), B(v_i) + B(v_j) \leq U(1) \}.
\end{equation}
where $B(v)$ is the weight of vertex $v$. Only vertex pairs satisfying capacity constraints are eligible for matching.

The pseudo code of the overall algorithm \ref{framework} is given below.

\begin{breakablealgorithm}
\caption{Multilevel MST-based hypergraph partitioning framework}
\label{framework}
\begin{algorithmic}[1]
\REQUIRE Hypergraph $\mathcal{H}$, Vertex Weights $B$, Vertex Set $\mathcal{V}$, Number of partitions $k$
\ENSURE Feasible partitioning result $\Pi^{\text{opt}}$

\STATE /* \textbf{Multilevel Coarsening Phase} */
\STATE The iteration starts from $i=0$
\WHILE{$|\mathcal{V}| > 625k$ or $M = \emptyset$ or $|\mathcal{V}|^{i} > 0.8*|\mathcal{V}|^{i-1}$ or $i < 20$}
    \STATE Compute edge contraction score $C(v_i, v_j)$ according to Eq.~\eqref{eq:edge-score}
    \STATE Construct matching set $M$ based on Eq.~\eqref{eq:matching-set}
    \STATE Contract matched vertex pairs and update vertex and hyperedge structures
    \STATE Generate coarsened hypergraph $\mathcal{H}^{(l)}$, $i++$
\ENDWHILE

\STATE /* \textbf{Initial Partitioning Phase} */
\STATE Expand the coarsened hypergraph $\mathcal{H}^{(l)}$ into a graph via \emph{clique expansion}.
\STATE Generate $s$ initial partition candidates $\Pi^{(l_{1})}$, $\Pi^{(l_2)}$, $\dots$, $\Pi^{(l_s)}$ \COMMENT{see Section~\ref{sect4}}
\STATE Perform optimization and enhancement on each partition $\Pi^{(l_i)}$\COMMENT{see Section~\ref{sect5}}
\STATE Select the best initial partition: $\Pi^{(l_{\text{opt}})} \gets \arg\min \text{cutsize}(\Pi^{(l_i)})$

\STATE /* \textbf{Uncoarsening and Refinement Phase} */
\FOR{each level $r = l \to 1$}
    \STATE Project solution to the finer level: $\Pi^{(r_{\text{opt}})} $
    \STATE Refine using $k$-way FM algorithm: $\Pi^{(r_{\text{opt}})} $
\ENDFOR
\end{algorithmic}
\end{breakablealgorithm}

\section{Modeling and solving multi-objective optimization problems} \label{sect3}

\subsection{Modeling multi-objective initial partitioning}\label{sect31}

The $k$-way graph partition problem can be written as
\begin{align}\label{eg3.11}
 \nonumber \mathop {\min } \ \ &\frac{1}{2}\langle L,XX^\top\rangle\\
 s.t. \ \  & X\in\{0,1\}^{n\times k},
\end{align}
where $X$ is a partition matrix containing only elements 0 and 1.
For example, consider the following matrix representation,
\[
X =
\begin{bmatrix}
0 & 0 & 0 & 1\\
1 & 0 & 0 & 0\\
0 & 1 & 0 & 0\\
0 & 0 & 0 & 1\\
0 & 0 & 1 & 0\\
0 & 0 & 1 & 0
\end{bmatrix},
\]
where each row corresponds to a vertex. For instance, the $i$-th row represents vertex $i$;
each column corresponds to a group, meaning the $j$-th column represents group $j$.
An entry $X(i, j) = 1$ indicates that vertex $i$ belongs to group $j$; otherwise, the value is 0.
Then (\ref{eg3.11}) equivalent to,
\begin{align}\label{eg3.12}
 \nonumber \mathop {\min } \ \ &-\frac{1}{2}\langle A,XX^\top\rangle\\
 s.t. \ \  & X\in\{0,1\}^{n\times k}.
\end{align}
We define \(\overline{A} = \text{Diag}\left(\sum_{j=0}^{n} a_{ij} \right) + A\) such that the adjacency matrix \( \overline{A} \) is diagonally dominant, making it a positive definite matrix.
This adjustment only adds a constant term to the previous objective function and does not change the original problem. For convenience, let \(f(V)= -\frac{1}{2}\langle \overline{A}, XX^\top\rangle\).

Consider the constraint \textbf{b},
let \( G_u \) and \( G_w \) denote the Laplacian matrices of the unit and weighted complete graphs, respectively.
Due to the following identities:
\begin{align}
&\frac{1}{2} \left\langle G_u, XX^\top \right\rangle
= \sum_{1 \le i < j \le k} |\Pi_i| \cdot |\Pi_j| \notag,\\
&\frac{1}{2} \left\langle G_w, XX^\top \right\rangle
= \sum_{1 \le i < j \le k}
\left( \sum_{X \in \Pi_i} B_v \right) \cdot \left( \sum_{u \in \Pi_j} B_u \right) \notag,
\end{align}

According to Lemma \ref{MAXBalanced}, maximizing the two expressions above drives the partition toward a balanced state. Consequently, this objective function can serve as an effective surrogate for the strict balance constraint \textbf{b}. For clarity, we define the following two functions:
\[
g_u(X) = - \frac{1}{2} \langle G_u, XX^\top \rangle, \quad \text{and} \quad g_w(X) = - \frac{1}{2} \langle G_w, XX^\top \rangle.
\]


Based on the above analysis, we now propose a multi-objective optimization model for generating an initial $k$-way partition. This formulation simultaneously considers three criteria:
\begin{itemize}
    \item The original objective function \(f(X)\), which may encode application-specific requirements such as cut-size.
    \item The function \(g_u(X)\), which promotes balanced partition sizes by penalizing imbalance in the number of vertices across blocks.
    \item The function \(g_w(X)\), which encourages balance in vertex weights, ensuring equitable distribution of capacity.
\end{itemize}

The resulting multi-objective model is as follows:
\begin{align}\label{eg3.2}
 \nonumber \mathop {\min } \ \ &f(X),\: g_{u}(X),\: g_{w}(X)\\
 \text{s.t.} \quad & X \in \{0,1\}^{n \times k}.
\end{align}

Here, the constraint \(X \in \{0,1\}^{n \times k}\) ensures that each vertex is assigned to exactly one of the \(k\) partitions. The multi-objective formulation allows trade-offs between the primary partitioning goal and the enforcement of balance, making the initial solution more robust and feasible for subsequent refinement stages.

\subsection{Solving models}\label{sect23}
A common approach to solving multi-objective optimization problems is to combine multiple objectives into a single objective by assigning appropriate weights to each goal, thereby optimizing the overall solution. The objective functions of the above two models can be expressed respectively as,
\begin{align}
\nonumber F_{C_{1}}(X)&= \lambda_{1}f(X) + (1-\lambda_{1})(\lambda_{2}g_{u}(X) + (1 - \lambda_{2})g_{w}(X)),\\
             &=\langle -C_{1} ,XX^{T} \rangle.\label{modle1}
\end{align}
where \( C_{1} = \lambda_{1}\overline{A} + (1-\lambda_{1})(\lambda_{2}G_{u} + (1 - \lambda_{2})G_{w} \)), $\lambda_{i}\in[0,1]$.
Then, (\ref{eg3.2}) can be represented as the following integer programming problem.
\begin{align} \label{eg3.4}
 \nonumber \mathop {\min } \ \ &F_{C_{1}}(X)\\
 s.t.\ \ & X\in\{0,1\}^{n\times k}.
\end{align}

We relax Equation~\eqref{eg3.4} into the following model:
\begin{align} \label{eg3.7}
 \nonumber \mathop {\min } \ \ & F_{C_{1}}(X)\\
 s.t.\ \ & diag(XX^\top) = \mathbf{1},\\
  \nonumber & X \in \mathbb{R}^{n \times k}.
\end{align}

For notational convenience, we define the feasible set
\[
\mathcal{S} := \left\{ X \in \mathbb{R}^{n \times k} \,\middle|\, \mathrm{diag}(XX^\top) = \mathbf{1} \right\},
\]
and rewrite the problem (\ref{eg3.7}) as:
\begin{align}\label{eg3.8}
\min_{X \in \mathbb{R}^{n \times k}} \quad & F_{C_{1}}(X) \\
\text{s.t.} \quad & X \in \mathcal{S}. \nonumber
\end{align}

We encode the constraint using the indicator function,
\[
g(X) = \delta_{\mathcal{S}}(X) =
\begin{cases}
0, & \text{if } \|X_{i,:}\|_2 = 1,\ \forall i, \\
+\infty, & \text{otherwise}.
\end{cases}
\]
Therefore, this problem is transformed into a composite optimization problem
\begin{equation} \label{comopt2}
\min_{X \in \mathbb{R}^{n \times k}} \ F_{C_{1}}(X) + g(X).
\end{equation}

Since \( g(X) \) is an indicator function over the set \( \mathcal{S} \), its proximal operator reduces to the projection onto \( \mathcal{S} \), which can be applied row-wise as:
\[
\text{Proj}_{\mathcal{S}}(X_{i,:}) =
\begin{cases}
\frac{X_{i,:}}{\|X_{i,:}\|_2}, & \text{if } \|X_{i,:}\|_2 \neq 0, \\
(1, 0, 0, \ldots, 0), & \text{if } \|X_{i,:}\|_2 = 0.
\end{cases}
\]
This projection ensures that each row \( X_{i,:} \) lies on the unit Euclidean sphere, thereby maintaining feasibility with respect to the constraint \( \mathrm{diag}(XX^T) = \mathbf{1} \) throughout the iterations.

We now verify that the objective function satisfies the K\L{} property, which is essential for the convergence analysis. According to~\cite{KL2010}, any proper lower semicontinuous \textit{semialgebraic} function satisfies the K\L{} property.
In our case, \( F_{C_1}(X) \) is a quadratic polynomial and hence semialgebraic. The constraint set \( \mathcal{S}\) is defined by polynomial equalities, so it is also semialgebraic. Consequently, the indicator function \( g(X) \) is semialgebraic, and so is the composite objective function \( F_{C_1}(X) +  g(X) \). Therefore, the full objective satisfies the K\L{} property.

Thus, we adopt \textbf{modAPG\_nc}~\cite{wang2024class} to solve the nonconvex constrained problem~\eqref{comopt2}.
The \textbf{modAPG\_nc} algorithm is specifically designed for nonsmooth and nonconvex composite optimization problems.
It incorporates an adaptive nonmonotone stepsize strategy, extrapolation, and a safeguard mechanism that resets the extrapolation parameter to zero at selected iterations.
Under the K\L{} property, theoretical results in~\cite{wang2024class} guarantee local convergence of the iterates, which applies directly to our setting.

We extend \textbf{modAPG\_nc} to the matrix setting \( \mathbb{R}^{n \times k} \), where the generated clusters correspond to stable points of the objective. The algorithm employs an adaptive stepsize strategy, which we further accelerate by providing a good initial stepsize estimate. Given an initial point \( X_0 \), the initial stepsize \( \alpha_0 \) is computed as:
\begin{align}\label{eg3.8}
\alpha_{0} = \frac{\|X_0 - X_1\|}{\|\nabla F_{C_1}(X_0) - \nabla F_{C_1}(X_1)\|},
\end{align}
where \( X_1 = \text{Proj}_{\mathcal{S}}(\nabla F_{C_1}(X_0)) \).
The remaining parameters follow the default settings in \cite{wang2024class}. The pseudocode of the extended \textbf{modAPG\_nc} algorithm \ref{modAPG} is provided below.

\begin{breakablealgorithm}
\caption{\textbf{modAPG\_nc}: Modified accelerated proximal gradient for nonconvex case}
\label{modAPG}
\small
\begin{algorithmic}[1]
\REQUIRE Initial point $X_0 = X_{-1} \in \mathbb{R}^{n\times \tilde{k}}$, parameters $0 < \mu_1 < \mu_0 < 1$, $0 < \delta_1 < \delta_2 < \frac{1 - \mu_0}{2 \alpha_1}$, extrapolation factors $\beta_k \in [0, 1]$, control parameter $\eta \in (0,1)$, sequence $E_k = \frac{1}{k^{1 + \tilde{p}}}$ with $\tilde{p} > 0$, $\epsilon=10^{-3}$.
\ENSURE Final output $X^*$
\STATE Set $c_0 = F_{C_{1}}(X_0)$, $q_0 = 1$, $k=0$. The step size $\alpha_0$ is calculated by (\ref{eg3.8})
\STATE Compute
\(error=\left\| \frac{ {X}_{k+1} -  {X}_k}{\alpha} + \left( \nabla F_{C_{1}}( {X}_{k+1}) - \nabla F_{C_{1}}( {X}_k) \right) \right\|_\infty\)
\WHILE {$error > \epsilon$}
 \STATE /* \textbf{Calculate the step size $\alpha_{k+1}$} */
    \IF{$2 \left[ F_{C_{1}}(X_k) - F_{C_{1}}(X_{k-1}) - \langle \nabla F_{C_{1}}(X_{k-1}), X_k - X_{k-1} \rangle \right] > \frac{\mu_0}{\alpha_k} \| X_k - X_{k-1} \|^2$}
        \STATE $\alpha_{k+1} \gets \mu_1 \cdot \frac{\| X_k - X_{k-1} \|^2}{2 \left[ F_{C_{1}}(X_k) - F_{C_{1}}(X_{k-1}) - \langle \nabla F_{C_{1}}(X_{k-1}), X_k - X_{k-1} \rangle \right]}$
    \ELSE
        \STATE $\alpha_{k+1} \gets \alpha_k + \min(1, \alpha_k) \cdot E_k$
    \ENDIF
    \STATE $y_{k+1} \gets X_k + \beta_k (X_k - X_{k-1})$, $z_{k+1} \gets \text{Proj}_{\mathcal{S}}(y_{k+1}-\alpha_{k+1}\nabla F_{C_{1}}(y_{k+1}))$
\STATE Compute:
\begin{equation*}
\begin{split}
\Phi_1 &\gets \| z_{k+1} - y_{k+1} \|^2 + \| z_{k+1} - x_k \|^2 \\
       &\quad - \left(1 + \frac{\sigma}{k^r} \right) \| y_{k+1} - X_k \|^2, \\
\Phi_2 &\gets \delta_1 \| z_{k+1} - X_k \|^2 \\
       &\quad - \delta_2 \left( \| z_{k+1} - y_{k+1} \|^2 + \| z_{k+1} - X_k \|^2 - \| y_{k+1} - X_k \|^2 \right).
\end{split}
\end{equation*}
     \STATE /* \textbf{Update iteration $X_{k+1}$} */
    \IF{ $\Phi_1 \geq 0$ and $F_{C_1}(z_{k+1}) \leq \min(F_{C_1}(X_k) + \Phi_2,\ c_k)$ }
        \STATE $X_{k+1} \gets z_{k+1}$
    \ELSE
        \STATE $X_{k+1} \gets \text{Proj}_{\mathcal{S}}(X_k-\alpha_{k+1}\nabla F_{C_{1}}(X_{k+1}))$
    \ENDIF
    \STATE Update $q_{k+1} \gets 1 + \eta q_k$, $c_{k+1} \gets \frac{ \eta q_k c_k + F_{C_1}(X_{k+1}) }{ q_{k+1} }$
    \STATE Update $error$, $E_{k}$, $k++$
\ENDWHILE
\end{algorithmic}
\end{breakablealgorithm}

At each iteration, the algorithm performs extrapolation to obtain a search point, followed by a projection step onto the feasible set $\mathcal{S}$ to satisfy the unit norm constraint.
The stepsize is adaptively adjusted using a nonmonotone scheme based on the observed objective decrease, and the extrapolation parameter $\beta_k$ is controlled such that it can be reset (i.e., $\beta_k = 0$) when the search direction is not reliable. Two acceptance conditions $\Phi_1$ and $\Phi_2$ are checked to ensure both descent and stability. When violated, the method performs a fallback proximal step using the current iterate.
We further incorporate a weighted average mechanism $c_k$ to control the nonmonotonicity of the update. The algorithm terminates once the scaled update and gradient change fall below a pre-specified threshold.
The time complexity per iteration of the algorithm is $O(n)$.

Moreover, by setting different parameters in Equation~\eqref{comopt2}, multiple $k$-dimensional information matrices for the vertices can be obtained by solving with Algorithm \textbf{modAPG\_nc}. Each of these matrices can then be used to construct a MST, from which a corresponding balanced partition can be derived.
\section{Initial partitioning via minimum spanning tree}\label{sect4}
This section introduces how to transform the vertex embeddings generated by Algorithm \textbf{modAPG\_nc} into a graph partitioning problem. Based on the size of the coarsened hypergraph, we design two different partitioning strategies.

For small-scale graphs, we first construct an adjacency graph by applying a distance threshold, then compute the MST using Prim's algorithm, and finally perform pruning and clustering on the tree to obtain a $k$-way partition.

For large-scale graphs, a subset of representative vertices is selected to construct the MST and perform pruning-based clustering. The remaining unselected vertices are then assigned to the nearest cluster based on their distance to the corresponding subtree cluster centers, resulting in a complete $k$-way partition.

In the multilevel partitioning framework, the original hypergraph $\mathcal{H} = (\mathcal{V}, \mathcal{E})$ is iteratively coarsened according to the condition specified on line 3 of Algorithm \ref{framework}.
 The choice of method for generating initial partitions is guided by the empirical size of the coarsened hypergraph, allowing a flexible trade-off between efficiency and quality.

\subsection{Method for generating feasible partitioning in small-scale datasets}
For small-scale datasets, we propose a method to generate feasible initial partitions, as illustrated in Figure~\ref{MSTer}.

Subfigure (a) depicts an incomplete graph constructed from the feature matrix; Subfigure (b) shows the MST generated using the classical Prim algorithm. The core idea of the Prim algorithm is to start from an arbitrary vertex \(v_0\), and iteratively connect the vertex closest to the current subtree. Specifically, the algorithm first connects \(v_0\) to its nearest vertex \(v_1\), forming the initial subtree \(T_1\); then connects \(T_1\) to its nearest vertex \(v_2\), forming \(T_2\); and continues this process until all vertices are included in the spanning tree.

Subfigure~(c) illustrates the partitioning of the MST into $p$ subtrees based on edge weight characteristics, where $p$ is typically greater than $k$. Finally, Subfigure~(d) shows how these clusters are merged to form the final $k$-way partition.

\begin{figure}
  \centering
  \includegraphics[width=12cm, height=9cm]{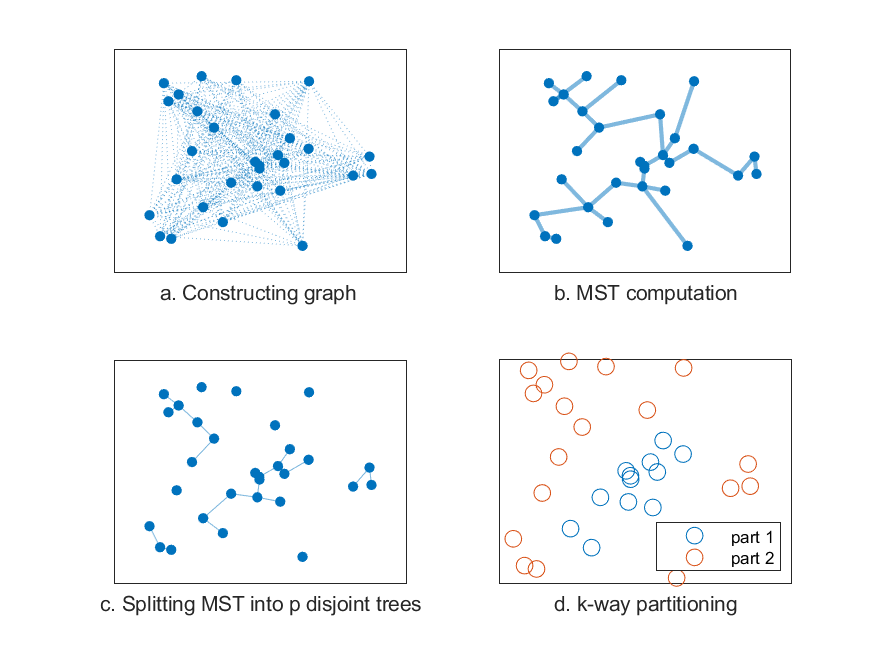}
  \caption{Feasible point generation method for small-scale datasets.}
  \label{MSTer}
\end{figure}

\subsubsection{Constructing a sparse graph and generating a minimum spanning tree}\label{smallMST1}

In this subsection, we introduce an enhanced MST algorithm built upon the classical Prim’s algorithm, which incorporates structural information derived from the feature space to improve performance. Specifically, instead of using a traditional graph adjacency matrix or an explicitly weighted graph, the input is a normalized feature matrix \( X \in \mathbb{R}^{n \times k} \), where each row \( X[i, :] \) represents the unit feature vector of the \( i \)-th sample. The similarity between samples is measured by the dot product of their corresponding feature vectors:
\[
s_{ij} = X[i, :] \cdot X[j, :]^\top,
\]
For any pair \( (i, j) \), if the similarity \( s_{ij} \) exceeds a predefined threshold \( \tau \), an edge is established between nodes \( i \) and \( j \). The weight of this edge is defined as
\[
w_{ij} = 1 - s_{ij},
\]
which is mathematically equivalent to half of the squared Euclidean distance between the normalized feature vectors. This thresholding approach effectively filters out irrelevant edges, thus avoiding the computational redundancy and high complexity caused by constructing a complete graph, while ensuring the resulting graph is sparse and representative. In summary, although the edge selection is based on similarity, the algorithm fundamentally constructs a MST in Euclidean space.

Based on the feature matrix and a predefined threshold, a non-complete graph $G$ can be constructed. When generating the MST on this graph, the following theorem holds:

\begin{theorem}
Let $G = (V, E, w)$ be a connected weighted graph. If the $k$ heaviest edges are removed from $G$, resulting in a new graph $G' = (V, E', w)$ that remains connected, then the MST $T'$ of $G'$ is also one of the MSTs of the original graph $G$.
\end{theorem}
\begin{proof}
The theorem is proved using mathematical induction.

Base case ($k=1$):
Assume the heaviest edge $a_1$ is removed from $G$, resulting in a new graph $G_1 = G \setminus a_1$, and suppose $G_1$ remains connected.
Since $a_1$ is the edge with the maximum weight in $G$, and removing it does not disconnect the graph, it must belong to some cycle $C$ in $G$ (otherwise, its removal would disconnect the graph).
According to the \textit{Cycle Property} (Theorem~\ref{Cycle Property}),
$a_1$ is the heaviest edge in the cycle $C$,
thus it does not belong to any MST of $G$.
Therefore, the MST of $G_1$ is also a MST of $G$.

Inductive hypothesis:
Assume that the statement holds for $k = t$. That is, after removing the $t$ heaviest edges $\{a_1, \dots, a_t\}$ from $G$, and assuming the resulting graph $G_t$ remains connected, then any MST of $G_t$ is also a MST of $G$.

Inductive step ($k = t+1$):
Now consider removing the $(t+1)$-th heaviest edge $e_{t+1}$ from $G_t$, resulting in $G_{t+1} = G_t \setminus a_{t+1}$, and suppose $G_{t+1}$ is still connected.
Since $e_{t+1}$ is the heaviest edge in $G_t$ and its removal does not disconnect the graph, it must lie on some cycle in $G_t$.

Again, by the \textit{Cycle Property} (Theorem~\ref{Cycle Property}),
$a_{t+1}$ is the heaviest edge in that cycle, and thus cannot appear in any MST of $G_t$.
Therefore, the MST of $G_{t+1}$ is the same as that of $G_t$.
From the inductive hypothesis, the MST of $G_t$ is also a MST of $G$.
Hence, the MST of $G_{t+1}$ is also a MST of $G$.

By the principle of mathematical induction, the conclusion holds for all $k \geq 1$.
\end{proof}

\begin{theorem}
Let $G = (V, E, w)$ be a connected, weighted, undirected graph, and let $T$ be a MST of $G$. Suppose the $k$ heaviest edges in $G$ are removed, resulting in a new graph $G' = (V, E')$, which becomes disconnected and consists of $m$ connected components, where $m \leq k + 1$. Let $F = (V, E')$ denote the resulting forest.

Then:
\begin{itemize}
\item[1)] Each connected component $F_i$ of $F$ is entirely contained within some subtree $T_i$ of the MST $T$.
\item[2)] If we select the $m - 1$ lightest edges among the removed $k$ edges that connect different components of $F$, and add them to $F$ to form a new graph $T'$, then $T'$ is also a MST of $G$.
\end{itemize}
\end{theorem}

\begin{proof}
\textbf{1)} Let $E_H \subseteq E$ denote the set of the $k$ heaviest edges in $G$, and define $E' = E \setminus E_H$. Let $F = (V, E')$ be the resulting forest with connected components $F_1, \dots, F_m$, where $m \leq k + 1$.

Let $T = (V, E_T)$ be a MST of $G$. Consider the subgraph $T' = (V, E_T')$ where $E_T' = E_T \setminus (E_T \cap E_H)$. That is, we remove from the MST all edges that also belong to $E_H$.

Since $T$ is a tree with $|V| - 1$ edges, removing $r \leq k$ edges from it yields at most $r + 1 \leq k + 1$ connected components. Thus, $T'$ is a forest consisting of subtrees $T_1, \dots, T_{m'}$ with $m' \leq k + 1$.

We claim that every connected component $F_i$ of $F$ is entirely contained in some $T_j$. Suppose not. Then there exist vertices $u \in V(T_j)$ and $v \in V(T_\ell)$ ($j \ne \ell$) such that $u, v \in V(F_i)$ and are connected in $F$ via a path $P_{u,v} \subseteq E'$.

However, since $T$ is a MST, the unique path between $u$ and $v$ in $T$ must pass through an edge in $E_T \cap E_H$. The existence of an alternate path in $F$ not using any edge from $E_H$ contradicts the minimality of $T$, since one could construct a spanning tree of lower or equal weight by replacing the heavy edge in $T$ with a lighter one from $P_{u,v}$. This contradicts the optimality of $T$.

Therefore, each $F_i$ must lie entirely within a subtree $T_j$.

\vspace{1em}
\textbf{2)} Let $\mathcal{C}$ be the set of edges in $E_H$ that connect different components of $F$. Select the $m - 1$ lightest edges $\{a_1, \dots, a_{m-1}\} \subseteq \mathcal{C}$, and define:
\[
T' = F \cup \{a_1, \dots, a_{m-1}\}.
\]
Since $F$ contains $|V| - m$ edges and we add $m - 1$ edges, $T'$ contains $|V| - 1$ edges. Moreover, these $m - 1$ edges reconnect the $m$ components of $F$, so $T'$ is connected and acyclic — i.e., a spanning tree.

To prove that $T'$ is a MST, note that the selected edges $\{a_1, \dots, a_{m-1}\}$ are the lightest edges crossing each cut between components. By the \emph{Cut Property} of MSTs, these edges must appear in every MST of $G$. Hence, $T'$ contains only edges from $F$ and cut-minimum edges, so no spanning tree can have smaller total weight.

Therefore, $T'$ is also a MST of $G$.
\end{proof}

\subsubsection{Generating $k$-way partitioning via pruned clustering}
Next, we introduce the partitioning process based on the MST, as illustrated in Subfigures~(c) and~(d) of Figure~\ref{MSTer}. The corresponding pseudocode is provided in Algorithm~\ref{Alg:MSTer}. This method performs clustering by pruning the MST.

First, the edges of the MST are sorted in descending order by weight, and the top $p - 1$ heaviest edges are removed. This operation divides the original MST into $p$ connected clusters $\mathcal{C} = \{C_1, \dots, C_p\}$.
Then, the total vertex weight $\mathcal{S}$ of each cluster is computed, and the $k$ clusters with the largest weights are selected as the initial partition set $\mathcal{P}$.
The remaining $p-k$ clusters are assigned one by one to the closest partition based on feature center distance. If the assignment satisfies the resource constraint after merging, it is accepted directly; otherwise, the cluster is assigned to the partition with the most available resources to promote load balancing.

This process prioritizes structural compactness while also considering resource balance.
It is worth noting that some partitions may not fully satisfy strict resource constraints at this stage. Feasibility correction and further optimization will be addressed in Section \ref{sect5}.

\begin{breakablealgorithm}
\caption{Balanced partitioning via MST-based edge pruning}
\label{Alg:MSTer}
\begin{algorithmic}[1]
\REQUIRE MST $T$, Target Clusters $p$,
   Vertex Weights $B \in \mathbb{R}^n$,
  Resource Limits $U \in \mathbb{R}^k$
\ENSURE  Partitions $\mathbf{\Pi} = \{\Pi_1,...,\Pi_k\}$

\STATE \textbf{/* Initial clustering via maximum-weight edge pruning on the MST */}
\STATE Sort edge set $T$ by edge weight in descending order \COMMENT{Sort edges in descending order to remove the heaviest ones}
\STATE Remove the top $p-1$ edges from $T$ to obtain $p$ connected components (clusters) $\mathcal{C} = \{C_1, \dots, C_p\}$

\STATE Compute cluster weights: $\mathcal{S} \gets \left\{ \sum_{j \in C_i} B_j \right\}_{i=1}^{p}$ \COMMENT{Total weight of each cluster}
\STATE $\mathcal{P} \gets \text{Select\_Top}_k(\mathcal{C}, \mathcal{S})$ \COMMENT{Select the $k$ heaviest clusters}
\STATE \textbf{/* Merge remaining clusters into $k$ approximately balanced partitions */}
\FOR{$C_i \in \mathcal{C} \setminus \mathcal{P}$}
    \STATE $j^* \gets \underset{j}{\text{argmin}} \|\mu(C_i) - \mu(\mathcal{P}_j)\|$ \COMMENT{Find nearest partition center}
    \IF{$\mathcal{S}(\mathcal{P}_{j^*}) + \mathcal{S}(C_i) \leq U_{j}$}
        \STATE $\mathcal{P}_{j^*} \gets \mathcal{P}_{j^*} \cup C_i$ \COMMENT{Merge into nearest feasible partition}
    \ELSE
        \STATE $\mathcal{P}_{j'} \gets \underset{j}{\text{argmin}}\:\mathcal{S}(\mathcal{P}_j)$ \COMMENT{Merge into partition with minimal load}
    \ENDIF
\ENDFOR
\STATE $\Pi \gets \mathcal{P}$
\end{algorithmic}
\end{breakablealgorithm}

The parameter $p$ in the algorithm controls the number of initial clusters generated during pruning, which directly affects the clustering quality and the granularity of the final partitioning.
If $p$ is set too large, excessive edge removals may occur, resulting in a large number of clusters. This not only increases the complexity of subsequent merging but may also degrade the partitioning quality.
Conversely, if $p$ is too small, the clustering becomes overly coarse and fails to capture the structural characteristics of the graph, thereby affecting the overall partitioning performance.

In this work, we adopt two typical values for $p$, namely $\sqrt{\frac{n}{2}}$ and $\frac{n}{5k}$, both of which ensure $p > k$, i.e., the number of clusters exceeds the target number of partitions. These values are adaptive to the graph size and vary accordingly with the number of vertices.

\subsection{Method for generating feasible partitions in large-scale datasets}
For large-scale datasets, directly constructing a complete MST incurs high computational complexity. To address this, we first select the top 20\% of nodes with the highest weights as representative nodes. A MST is then constructed among these representative nodes, and an initial partition is generated according to Algorithm~\ref{Alg:MSTer}. The remaining non-representative nodes are subsequently assigned to the partition of their nearest representative node using a clustering strategy, as illustrated in Figure~\ref{RNDGD}. The corresponding implementation is shown in Algorithm~\ref{ALG:RNDGD}.

\begin{figure}
  \centering
  \includegraphics[width=12cm, height=9cm]{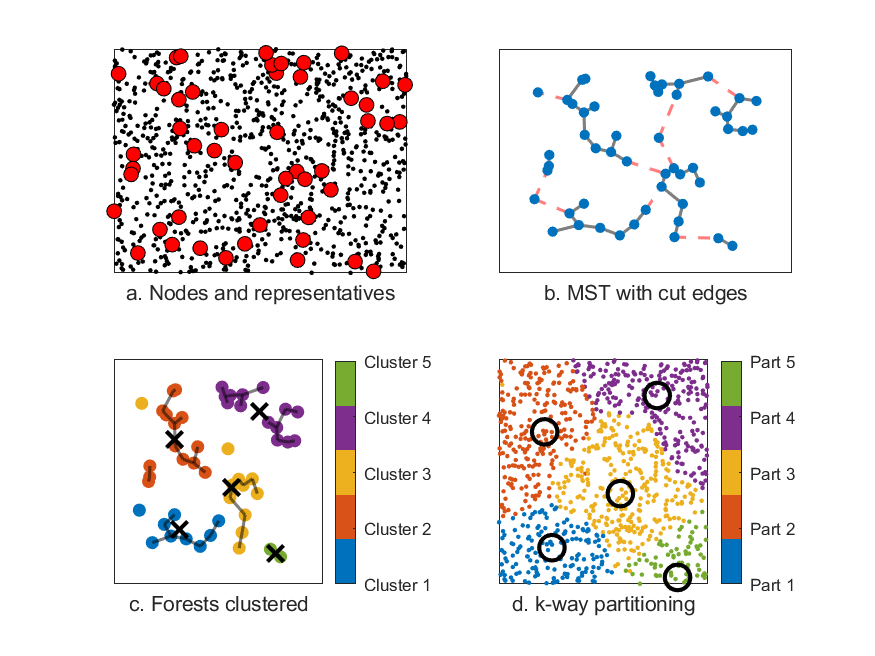}
  \caption{Method for generating feasible points in large-scale datasets.}
  \label{RNDGD}
\end{figure}

\begin{breakablealgorithm}
\caption{Representative nodes drive global division}
\label{ALG:RNDGD}
\begin{algorithmic}[1]
\REQUIRE Feature matrix $X \in \mathbb{R}^{n \times k}$, Vertex weights $B\in \mathbb{R}^n$, Partitions $k$, Target Clusters $p$, Resource Limits $U \in \mathbb{R}^k$
\ENSURE Balanced partitions $\Pi=\{\Pi_1,...,\Pi_p\}$

\STATE /* \textbf{Step 1: Minimum spanning tree over representative nodes}*/
\STATE Select $\mathcal{N} \gets \text{Top-} \lceil 0.2n \rceil \text{ vertices by } B$ \COMMENT{Select representative nodes}
\STATE Initialize $\text{visited}, \text{minDist}, \text{parent}, T$

\FOR{$k \gets 1$ to $m$}
    \STATE $u \gets  \mathrm{argmin}_{i \notin \text{visited}} \text{minDist}[i]$ \COMMENT{Greedy selection}
    \STATE $\text{visited}[u] \gets \text{True}$
    \STATE Add edge to $T$ if $\text{parent}[u] \neq -1$, with weight $w = 1 - \langle X_{\mathcal{N}[u]}, X_{\mathcal{N}[\text{parent}[u]]} \rangle$

\FOR{each unvisited node $v$}
        \STATE $s \gets X_{\mathcal{N}[u]} \cdot X_{\mathcal{N}[v]}^\top$
        \IF{$s > 0.2$ and $1 - s < minDist[v]$}
            \STATE $minDist[v] \gets 1 - s$, $parent[v] \gets u$
        \ENDIF
    \ENDFOR
\ENDFOR
\STATE /* \textbf{Step 2: Prune into $p$ clusters and merge them into$k$ balanced partitions} */
\STATE Sort edge set $T$ by edge weight in descending order 
\STATE Remove the top $p-1$ edges from $T$ to obtain $p$ connected components (clusters) $\mathcal{C} = \{C_1, \dots, C_p\}$
\STATE Compute cluster weights: $\mathcal{S} \gets \left\{ \sum_{j \in C_i} B_j \right\}_{i=1}^{p}$ \COMMENT{Cluster Weight Calculation}
\STATE $U_{adj} \gets (1+\varepsilon)\frac{\sum \mathcal{S}}{k} \cdot \mathbf{1}_k$ \COMMENT{Adaptive resource adjustment}
\FOR{each cluster $C_i \in \mathcal{C}$}
    \STATE $j^* \gets \underset{j}{\text{argmin}} \| {X}_v - \mu_j \| \text{ s.t. } \mathcal{S}_j + B(v) \leq U_{adj}$
    \STATE Assign to nearest feasible $\mathcal{P}_j^*$, else assign to least loaded
    \STATE Update $\mathcal{S}$
\ENDFOR

\STATE /* \textbf{Step 3: Assignment of remaining nodes for final partitioning} */
\FOR{each $v \in \mathcal{V \setminus \mathcal{N}}$}
\STATE $j^* \gets \underset{j}{\text{argmin}} \| {X}_v - \mu_j \| \text{ s.t. } \mathcal{S}_j + B(v) \leq U_j$
    \STATE Assign $v$ to nearest feasible $\mathcal{P}_{j^{*}}$, else to least loaded
    \STATE Update centroid $\mu_j$ and weight $\mathcal{S}_j$
\ENDFOR
\STATE $\Pi \gets \mathcal{P}$
\end{algorithmic}
\end{breakablealgorithm}

The algorithm first selects appropriate representative nodes to construct a MST, which is then pruned to form $p$ initial clusters. Subsequently, based on feature center distances and an improved resource constraint strategy, these initial clusters are merged into $k$ approximately balanced partitions. The remaining unassigned vertices are allocated to the nearest feasible cluster according to their feature vectors. The final output is a partitioning scheme that is both structurally compact and resource balanced. This method effectively avoids the high computational cost of directly constructing an MST over large-scale vertex sets by integrating clustering techniques to achieve efficient partitioning.

\section{Refinement strategies for initial partitions}\label{sect5}

This section presents strategies for optimizing and adjusting existing partitions.
For partitions that satisfy the constraints, we further improve the partition quality by introducing a bipartition refinement algorithm based on the MST.
For infeasible partitions, we propose a greedy move-and-swap strategy to restore feasibility.

\subsection{Modeling multi-objective improvements}
Similar to the modeling and solution approach for the initial partitioning, we consider three objectives: the minimum cut objective for graph partitioning, the equilibrium constraint, and the extraction of initial partition information.
For the first two objectives, the objective functions \( f(X) \) and \( g_w(X) \) remain the same as defined in Subsection~\ref{sect31}.
For the third objective, we first construct a complete $k$-partite graph based on the initial partition. And optimizing the objective \( g_{k}(X) = -\frac{1}{2}\langle K_{k}, XX^\top \rangle \) with \( K_{w} \) denoting the Laplacian matrix of the complete $k$-partite graph. We formulate a multi-objective 0-1 integer programming model as follows:
\begin{align}\label{eg3.3}
 \nonumber \min \quad & f(X),\: g_w(X),\: g_k(X) \\
 \text{s.t.} \quad & X \in \{0,1\}^{n \times k}.
\end{align}

In the above model, $X$ represents the partition indicator matrix, where each column corresponds to one of the $k$ partitions. Leveraging the orthogonality of $X$, we note that exchanging vertices between any two blocks does not affect the remaining blocks. Based on this observation, we perform $\lfloor \frac{k}{2} \rfloor$ pairwise optimizations, where each pair of partitions is optimized independently to improve the overall partition quality.

This formulation also reduces the $k$-dimensional assignment information of each vertex to a 2-dimensional representation for each pair, allowing us to solve multiple smaller subproblems. For each partition pair, we construct a multi-objective model:
\begin{align} \label{modle2}
\nonumber \min \quad & f(X), \: f_w(X), \: g_k(X) \\
   \text{s.t.} \quad & X \in \{0,1\}^{\bar{n} \times 2},
\end{align}
where $\bar{n}$ denotes the number of vertices involved in the current pair.

Similar to model~\eqref{modle1}, the multi-objective optimization is reformulated as a single-objective problem via a linear weighted sum with weights \(\xi_1, \xi_2 \in [0,1]\), resulting in the following formulation:
\begin{align}
\nonumber
F_{C_2}(X) &= \xi_1 f(X) + (1 - \xi_1) \bigl( \xi_2 g_w(X) + (1 - \xi_2) g_k(X) \bigr) \\
&= \langle -C_2, XX^\top \rangle,
\end{align}
where
\[
C_2 = \xi_1 \overline{A} + (1 - \xi_1) \left( \xi_2 G_w + (1 - \xi_2) K_w \right).
\]

Here, \(\xi_1\) balances the influence of the original graph structure and the initial partition, while \(\xi_2\) controls the trade-off between global structure (modeled by the complete graph) and local refinement (modeled by the complete bipartite graph).

Relaxing the binary constraint in model~\eqref{modle2} yields the following continuous optimization problem,
\begin{align} \label{modle2_relaxed}
 \min \quad & F_{C_2}(X) \\
 \text{s.t.} \quad & \operatorname{diag}(XX^\top) = \mathbf{1}, \nonumber \\
 & X \in \mathbb{R}^{\bar{n} \times 2}. \nonumber
\end{align}

Therefore, this problem is also transformed into a composite optimization problem, which is solved using the \textbf{modAPG\_nc} algorithm,
\begin{equation} \label{comopt3}
\min_{X \in \mathbb{R}^{\bar{n} \times 2}} \ F_{C_{2}}(X) + g(X).
\end{equation}

\subsection{Bipartition-based iterative optimization}

The balanced bipartition algorithm combining MST and clustering is based on constructing two cluster centers from the MST and partitioning all vertices into two groups according to these centers. Specifically, the algorithm first selects a subset of nodes with the largest weights as key nodes from the input data and builds a MST based on these nodes. Then, the top 20\% heaviest edges in the MST are considered as candidate cut edges. For each candidate edge, the graph is split into two subclusters, and the mean vector of each subcluster is computed as its cluster center. Each data point is assigned a partition label (1 or -1) based on its proximity to the two centers. Different cuts correspond to different cluster centers, thus producing multiple partitioning results.

During the evaluation phase, the algorithm checks whether each partition satisfies the given balance constraints (i.e., the total weights of the two partitions do not exceed \(U_1\) and \(U_2\), respectively) and computes the objective value \(y^\top L y\) for feasible partitions. Finally, the algorithm selects the partition with the minimum objective value that satisfies the constraints among all candidate cuts and different parameter settings \(k\) as the final result. The pseudocode of the algorithm is shown in Algorithm~\ref{alg:mst_partition}.
\begin{breakablealgorithm}
\caption{Optimization-driven graph partitioning based on MST decomposition }
\label{alg:mst_partition}
\begin{algorithmic}[1]
\REQUIRE
Feature matrix $X \in \mathbb{R}^{\bar{n} \times 2}$,
Vertex weights $B \in \mathbb{R}^{\bar{n}}$,
Resource Limits $U = (U_1,U_2)^\top$,
Graph Laplacian matrix $L \in \mathbb{R}^{\bar{n} \times \bar{n}}$, where $\bar{n}\leq n$
\ENSURE 2-way partitioning $\Pi$
\STATE Initialize $\mathbf{y} \leftarrow \mathbf{0}$, $\text{min\_obj} \leftarrow \infty$
\STATE Select $\mathcal{P}_{sel} \gets \text{Top-}\lceil 0.05n \rceil \text{ vertices by } B$ \COMMENT{Selection of key nodes}

\STATE /* \textbf{Build Minimum Spanning Tree} */
\STATE $T \leftarrow \text{PrimMST}(X, \mathcal{P}_{sel})$ \COMMENT{Using Euclidean distances}
\STATE $\mathcal{E}_{\text{cut}} \leftarrow \text{top}_m(\text{edges}(T))$ \COMMENT{Top $m=\lfloor 0.2k \rfloor$ heaviest edges}
\FOR{$\text{edge} \in \mathcal{E}_{\text{cut}}$}
    \STATE $(\mathcal{C}_1, \mathcal{C}_2) \leftarrow \text{split}(T, \text{edge})$ \COMMENT{Partition graph by removing edge}
    \STATE $\mathbf{c}_1 \leftarrow \text{mean}(X_{\mathcal{C}_1})$, $\mathbf{c}_2 \leftarrow \text{mean}(X_{\mathcal{C}_2})$ \COMMENT{Compute cluster centers}
    \STATE $\mathbf{y}_{\text{temp}} \leftarrow \text{sign}(\|X-\mathbf{c}_1\|^2 - \|X-\mathbf{c}_2\|^2)$ \COMMENT{Assign partition labels}
    \STATE /* \textbf{Check constraints} */
    \IF{$\mathbf{y}_{\text{temp}}^\top B \leq U_1$ \AND $-\mathbf{y}_{\text{temp}}^\top B \leq U_2$}
        \STATE $\text{obj} \leftarrow \frac{1}{4}\mathbf{y}_{\text{temp}}^\top L \mathbf{y}_{\text{temp}}$
        \IF{$\text{obj} < \text{min\_obj}$}
            \STATE $\text{min\_obj} \leftarrow \text{obj}$, $\mathbf{y} \leftarrow \mathbf{y}_{\text{temp}}$
        \ENDIF
    \ENDIF
\ENDFOR
\STATE Convert $\mathbf{y}$ into vertex partition $\Pi$ according to partition labels

\end{algorithmic}
\end{breakablealgorithm}

\subsection{Optimization strategies for existing partitions}\label{adjust1}

The following optimization algorithm is designed for a general $k$-way partitioning problem. It is based on iteratively reorganizing sub-partitions to improve the objective function, leveraging Algorithm~\ref{alg:mst_partition}. The detailed procedure is as follows:

\begin{enumerate}
    \item \textbf{Pairing sub-partitions.}
    Given an initial $k$-way partition $\{\Pi_1, \Pi_2, \ldots, \Pi_k\}$,
    the $k$ sub-partitions are sorted based on their connectivity strength and paired accordingly;
    this results in $\lceil k/2 \rceil$ pairs, with the last sub-partition left alone if $k$ is odd.

    \item \textbf{Parallel optimization of each pair.}
    For each pair $(\Pi_i, \Pi_j)$, the \textbf{modAPG\_nc} algorithm is applied to extract 2-dimensional vertex features from the objective model (\ref{modle2});
    then, Algorithm~\ref{alg:mst_partition} is used to repartition the combined sub-partition $\Pi_i \cup \Pi_j$ into $\{\Pi_i', \Pi_j'\}$.

    \item \textbf{Acceptance criterion and global merging.}
    The objective values of the new sub-partitions $(\Pi_i', \Pi_j')$ are compared with those of the original pair $(\Pi_i, \Pi_j)$;
    if the objective value decreases, the new partition is accepted, otherwise the original partition is retained;
    all sub-partitions are then merged to form the optimized $k$-way partition.
\end{enumerate}

For example, given an initial partition $\{\Pi_1, \Pi_2, \Pi_3\}$, suppose the pairs are $(\Pi_1, \Pi_2)$ and $(\Pi_3)$;
optimizing $\Pi_1 \cup \Pi_2$ yields $\{\Pi_1', \Pi_2'\}$;
merging results in $\{\Pi_1', \Pi_2', \Pi_3\}$;
this process iterates until the objective value no longer improves.

\section{Numerical experiments}\label{sect6}

To validate the effectiveness of our algorithms, we conduct numerical experiments using the ISPD98 VLSI circuit benchmark set with actual weights \cite{ISPD98} and Titan23 benchmarks with unit weights \cite{Titan23}. We compare our methods with two state-of-the-art hypergraph partitioners, KaHyPar \cite{kahypar2023high}, Mt-KaHyPar \cite{mtkahypar}, hMetis \cite{khmetis1999} and K-SpecPart \cite{kspecpart2022}. Our implementation is developed in C++ and compiled using g++ 13.3.0. All experiments are conducted on a server equipped with an Intel(R) Xeon(R) Gold 6248R CPU @ 3.00GHz, running the Ubuntu 24.04.2 LTS operating system.

\subsection{Parameters setting.}
In the partitioner hMetis, we denote the balance factor by UBfactor. To maintain consistency in the balance constraint, we enforce the following,
\begin{align*}
(1+\varepsilon)\frac{\sum_{v \in \mathcal{V}} B(v)}{k}=(\frac{50+\text{UBfactor}}{100})^{\log_{2}k}\sum_{v \in \mathcal{V}} B(v) 
\Longrightarrow
 \varepsilon = (\frac{50+\text{UBfactor}}{100})^{\log_{2}k}*k-1.
\end{align*}

Let UBfactor = 2 with hMetis, the balance factors for 2-way, 3-way, and 4-way partitions are set to $\varepsilon = 0.04$, $0.06$, and $0.08$, respectively for our algorithm. For partitions with more than four parts, we set $\varepsilon = 0.02$.
 As KaHyPar adheres to the constraints in this paper, identical balance factors are enforced across all $k$-way partitions.
Considering the influence of random seeds on KaHyPar and Mt-KaHyPar performance, we run it five times with seeds -1, 5, 10, 15, and 20, and report the average cutsize and CPU time. hMetis is run ten times to select the best partitioning result. For K-SpecPart, with the ub\_factor = 2, the results are derived from the published partition files.
All other parameters for partitioners are set to their default values\textsuperscript{1}.

\footnotetext[1]{For KaHyPar, the configuration file used is km1\_kKaHyPar\_sea20.ini. For Mt-KaHyPar, preset type=highest\_{quality}, t=48 and objective=km1.
For hMetis, the parameters are UBfactor = 2, Nruns = 10, CType=1, RType=1, Vcycle=1, Reconst=0 and dbglvl=0. For K-SpecPart, the partition files can be found at https://github.com/TILOS-AI-Institute/HypergraphPartitioning/tree/
main/K\_specpart\_solutions.
}

In our algorithm, when the parameters are fixed, the results remain consistent. Therefore, the algorithm only needs to be executed once for comparison purposes.
This study involves several parameter choices, including the clustering size during the coarsening phase, the weighting coefficients $\lambda_{i}$ and $\xi_{i}$ in the objective functions (\ref{eg3.7}), (\ref{modle2_relaxed}) used during initial partitioning, the number of initial partitions generated (\texttt{num\_Init}), and the number $p$ of edges pruned in the MST.
After coarsening, we apply size-dependent strategies for initial partitioning: Prim-based MST pruning (Algorithm~\ref{Alg:MSTer}) when $|\mathcal{V}| \leq 35,000$, and representative-node clustering (Algorithm~\ref{ALG:RNDGD}) otherwise. 

In addition, to conserve space in the tables, we use the following abbreviations for the partitioning tools: K for KaHyPar, H for hMetis, S for K-SpecPart, M for Mt-KaHyPar, and O for our proposed method. 
The corresponding parameter settings are summarized in Table~\ref{table3ParameSetting}, and a sensitivity analysis of these parameters will be conducted in subsequent sections.
\begin{table}[!htbp]
\centering
\caption{Configuration of parameter values.}
\label{table3ParameSetting}

\begin{tabular}{c|c|c|c}
\hline
parameter & Values & parameter & Values  \\
\hline
  $\lambda_{1}$ & 0.9,\:0.5,\:0.15,\:0.015 &$\xi_{1}$ & 0.5,\: 0.15  \\
  $\lambda_{2}$ & 1, \:0.9,\: 0.8 & $\xi_{2}$ & 1,\:0.8,\: 0.2  \\
  $p$& $\sqrt{\frac{n}{2}}$, \:$\frac{n}{5k}$&  \texttt{num\_Init} & 10 \\
\hline
\end{tabular}
\end{table}

\subsection{Results of ISDP98 benchmarks with actual weights.}
To enhance the generality of the problem formulation, we incorporate weight settings into our algorithm.
Table~\ref{ISPD98w234} presents a performance comparison between our algorithm, KaHyPar, hMetis, Mt-KaHyPar and K-SpecPart on the ISPD98 benchmark suite, evaluated under partition numbers \(k = 2, 3, 4\). 
Among the 18 test cases, the best result for each instance is highlighted in red. Additionally, the table reports average performance relative to KaHyPar for overall comparison.

\begin{table*}[htbp]
  \centering
  \caption{Comparison of partitioners and our algorithm on ISPD98 benchmarks for $k$ = 2, 3, 4.}
   \setlength{\tabcolsep}{2pt} 
 \resizebox{\linewidth}{!}{
    \begin{tabular}{||c||ccccc||ccccc||ccccc||}
    \hline
   \multirow{2}[4]{*}{ISPD98} & \multicolumn{5}{c||}{$k=2$}      & \multicolumn{5}{c||}{$k=3$}            & \multicolumn{5}{c||}{$k=4$} \\
\cline{2-16}          & \multicolumn{1}{c|}{K} & \multicolumn{1}{c|}{H} & \multicolumn{1}{c|}{S} & \multicolumn{1}{c|}{M } & O     & \multicolumn{1}{c|}{K} & \multicolumn{1}{c|}{H} & \multicolumn{1}{c|}{S} & \multicolumn{1}{c|}{M } & O     & \multicolumn{1}{c|}{K} & \multicolumn{1}{c|}{H} & \multicolumn{1}{c|}{S} & \multicolumn{1}{c|}{M } & \multicolumn{1}{c||}{O } \\
    \hline
    ibm01 & \textcolor[rgb]{ 1,  0,  0}{\textbf{215.0 }} & \textcolor[rgb]{ 1,  0,  0}{\textbf{215 }} & \textcolor[rgb]{ 1,  0,  0}{\textbf{215 }} & \textcolor[rgb]{ 1,  0,  0}{\textbf{215.0 }} & 218   & \textcolor[rgb]{ 1,  0,  0}{\textbf{365.8 }} & 380   & 446   & 368.8  & 368   & \textcolor[rgb]{ 1,  0,  0}{\textbf{354.4 }} & 461   & 369   & 354.8  & 370   \\
    ibm02 & 287.2  & 297   & 296   & 270.8  & \textcolor[rgb]{ 1,  0,  0}{\textbf{266 }} & 399.2  & 367   & 367   & 429.0  & \textcolor[rgb]{ 1,  0,  0}{\textbf{359 }} & 563.4  & 587   & 573   & 578.6  & \textcolor[rgb]{ 1,  0,  0}{\textbf{534 }} \\
    ibm03 & 867.8  & 771   & 957   & 848.0  & \textcolor[rgb]{ 1,  0,  0}{\textbf{695 }} & 1,372.8  & 1,250  & 1,277  & 1,297.2  & \textcolor[rgb]{ 1,  0,  0}{\textbf{1,216 }} & 1,574.2  & 1,912  & 1,913  & 1,593.0  & \textcolor[rgb]{ 1,  0,  0}{\textbf{1,487 }} \\
    ibm04 & \textcolor[rgb]{ 1,  0,  0}{\textbf{491.2 }} & 541   & 529   & 496.6  & 505   & 916.0  & 893   & \textcolor[rgb]{ 1,  0,  0}{\textbf{886 }} & 889.4  & 921   & 1,518.8  & 1,648  & 1,586  & 1,464.2  & \textcolor[rgb]{ 1,  0,  0}{\textbf{1,436 }} \\
    ibm05 & \textcolor[rgb]{ 1,  0,  0}{\textbf{1,720.8 }} & 1,747  & 1,721  & 1,722.0  & 1,773  & 2,898.4  & 3,025  & 3,106  & \textcolor[rgb]{ 1,  0,  0}{\textbf{2,860.4 }} & 3,189  & 3,387.6  & 3,648  & 3,681  & \textcolor[rgb]{ 1,  0,  0}{\textbf{3,386.6 }} & 3,460  \\
    ibm06 & 572.6  & 483   & 845   & 514.8  & \textcolor[rgb]{ 1,  0,  0}{\textbf{476 }} & 1,164.2  & \textcolor[rgb]{ 1,  0,  0}{\textbf{978 }} & 1,002  & 1,070.2  & 981   & 1,501.8  & 1,367  & 1,398  & 1,383.2  & \textcolor[rgb]{ 1,  0,  0}{\textbf{1,331 }} \\
    ibm07 & 788.8  & 792   & 803   & 756.8  & \textcolor[rgb]{ 1,  0,  0}{\textbf{743 }} & 1,300.2  & 1,482  & 1,351  & 1,306.6  & \textcolor[rgb]{ 1,  0,  0}{\textbf{1,212 }} & 2,117.0  & 2,150  & 2,074  & 2,018.6  & \textcolor[rgb]{ 1,  0,  0}{\textbf{1,899 }} \\
    ibm08 & 1,257.0  & 1,186  & \textcolor[rgb]{ 1,  0,  0}{\textbf{1,182 }} & 1,206.8  & 1,215  & 1,782.8  & 1,775  & 1,892  & 1,830.0  & \textcolor[rgb]{ 1,  0,  0}{\textbf{1,702 }} & 2,376.8  & 2,499  & 2,565  & 2,440.8  & \textcolor[rgb]{ 1,  0,  0}{\textbf{2,221 }} \\
    ibm09 & 519.2  & 523   & \textcolor[rgb]{ 1,  0,  0}{\textbf{519 }} & 555.6  & 539   & 1,142.8  & 1,303  & 1,380  & \textcolor[rgb]{ 1,  0,  0}{\textbf{1,120.2 }} & 1,126  & 1,554.4  & 1,595  & 1,577  & 1,513.4  & \textcolor[rgb]{ 1,  0,  0}{\textbf{1,478 }} \\
    ibm10 & 1,332.2  & 1,062  & 1,028  & 1,234.4  & \textcolor[rgb]{ 1,  0,  0}{\textbf{988 }} & \textcolor[rgb]{ 1,  0,  0}{\textbf{1,533.0 }} & 1,606  & 1,560  & 1,623.8  & 1,627  & 2,087.2  & 2,170  & 2,202  & 2,314.0  & \textcolor[rgb]{ 1,  0,  0}{\textbf{1,887 }} \\
    ibm11 & 767.0  & 780   & \textcolor[rgb]{ 1,  0,  0}{\textbf{763 }} & 775.0  & 791   & 1,512.8  & 1,548  & 1,537  & 1,524.2  & \textcolor[rgb]{ 1,  0,  0}{\textbf{1,437 }} & 2,127.4  & 2,161  & 2,119  & 2,136.4  & \textcolor[rgb]{ 1,  0,  0}{\textbf{1,939 }} \\
    ibm12 & \textcolor[rgb]{ 1,  0,  0}{\textbf{1,965.2 }} & 2,025  & 1,967  & 2,308.6  & 2,033  & 3,623.4  & 3,203  & 3,188  & 3,596.0  & \textcolor[rgb]{ 1,  0,  0}{\textbf{3,020 }} & 3,843.8  & 4,285  & 3,728  & 3,759.6  & \textcolor[rgb]{ 1,  0,  0}{\textbf{3,696 }} \\
    ibm13 & 1,053.2  & 894   & \textcolor[rgb]{ 1,  0,  0}{\textbf{846 }} & 1,005.2  & 880   & 1,691.0  & 1,673  & 1,659  & 1,648.6  & \textcolor[rgb]{ 1,  0,  0}{\textbf{1,462 }} & 2,038.2  & 2,097  & 2,028  & 2,007.4  & \textcolor[rgb]{ 1,  0,  0}{\textbf{1,824 }} \\
    ibm14 & 1,764.0  & \textcolor[rgb]{ 1,  0,  0}{\textbf{1,704 }} & 1,929  & 1,754.0  & 1,772  & 2,870.0  & 3,163  & 2,935  & \textcolor[rgb]{ 1,  0,  0}{\textbf{2,848.2 }} & 3,047  & \textcolor[rgb]{ 1,  0,  0}{\textbf{3,257.8 }} & 3,869  & 3,626  & 3,323.2  & 3,330  \\
    ibm15 & \textcolor[rgb]{ 1,  0,  0}{\textbf{1,986.6 }} & 2,123  & 2,474  & 2,361.2  & 2,004  & 4,042.4  & 3,933  & 3,866  & \textcolor[rgb]{ 1,  0,  0}{\textbf{3,822.2 }} & 3,862  & \textcolor[rgb]{ 1,  0,  0}{\textbf{4,888.2 }} & 4,938  & 5,005  & 5,025.8  & 4,922  \\
    ibm16 & 1,772.6  & 1,718  & 1,660  & \textcolor[rgb]{ 1,  0,  0}{\textbf{1,647.4 }} & 1,707  & 3,142.0  & 3,038  & 2,801  & 3,191.6  & \textcolor[rgb]{ 1,  0,  0}{\textbf{2,764 }} & 3,963.4  & 4,063  & 3,848  & 4,069.8  & \textcolor[rgb]{ 1,  0,  0}{\textbf{3,595 }} \\
    ibm17 & \textcolor[rgb]{ 1,  0,  0}{\textbf{2,281.8 }} & 2,535  & 2,301  & 2,294.8  & 2,335  & 3,638.4  & 4,017  & 3,757  & 3,678.4  & \textcolor[rgb]{ 1,  0,  0}{\textbf{3,549 }} & \textcolor[rgb]{ 1,  0,  0}{\textbf{4,733.2 }} & 6,153  & 5,977  & 5,043.8  & 4,914  \\
    ibm18 & 1,877.4  & 1,677  & \textcolor[rgb]{ 1,  0,  0}{\textbf{1,579 }} & 1,935.6  & 1,622  & \textcolor[rgb]{ 1,  0,  0}{\textbf{2,646.2 }} & 3,001  & 3,027  & 2,722.2  & 2,682  & 3,278.0  & 3,731  & 3,650  & 3,282.8  & \textcolor[rgb]{ 1,  0,  0}{\textbf{3,238 }} \\
    all   & 21,519.6  & 21,073  & 21,614  & 21,902.6  & 20,562  & 36,041.4  & 36,635  & 36,037  & 35,827.0  & 34,524  & 45,165.6  & 49,334  & 47,919  & 45,696.0  & 43,561  \\
    mean  & 1.000  & 1.113  & 1.019  & 1.002  & 0.949  & 1.000  & 1.013  & 1.012  & 0.997  & 0.956  & 1.000  & 1.090  & 1.050  & 1.005  & 0.956 \\
    \hline
    \end{tabular}}%
  \label{ISPD98w234}
\end{table*}%

Experimental results demonstrate that across different partition counts ($k=2, 3, 4$), the Ours algorithm significantly outperforms other partitioners in terms of partition quality (cut size). Compared to KaHyPar, the cut size is reduced by approximately 5\% on average, with a maximum reduction of 26\%.
Our algorithm achieves a 44\% improvement in solution quality over K-SpecPart, 20\% over Mt-KaHyPar, and 22\% over hMetis on certain instances. Furthermore, the Ours algorithm delivers optimal or near-optimal partitioning results in most instances. Notably, in the 4-way partitioning scenario, over two-thirds of the results surpass those of KaHyPar, and all results exceed those of hMetis and K-SpecPart, underscoring its strong capability in high-quality partitioning tasks.

 \subsection{Results of Titan23 benchmarks}
The Titan23 benchmarks are significantly larger and more challenging compared to the ISPD98 benchmarks. The quality comparisons between our algorithm and other partitioners are presented in Table~\ref{tab:Titancutk234}.

\begin{table}[htbp]
  \centering
  \caption{Comparison of partitioners and our algorithm on Titan23 benchmarks for $k$ = 2, 3, 4.}
   \setlength{\tabcolsep}{2pt} 
 \resizebox{\linewidth}{!}{
    \begin{tabular}{||c||ccccc||ccccc||ccccc||}
    \hline
    \multirow{2}[4]{*}{Titan23} & \multicolumn{5}{c||}{$k=2$}      & \multicolumn{5}{c||}{$k=3$}            & \multicolumn{5}{c||}{$k=4$} \\
\cline{2-16}          & \multicolumn{1}{c|}{K} & \multicolumn{1}{c|}{H} & \multicolumn{1}{c|}{S} & \multicolumn{1}{c|}{M } & O     & \multicolumn{1}{c|}{K} & \multicolumn{1}{c|}{H} & \multicolumn{1}{c|}{S} & \multicolumn{1}{c|}{M } & O     & \multicolumn{1}{c|}{K} & \multicolumn{1}{c|}{H} & \multicolumn{1}{c|}{S} & \multicolumn{1}{c||}{M } & O   \\
    \hline
    bitcoin\_miner & 2,480.8  & \textcolor[rgb]{ 1,  0,  0}{\textbf{1,512 }} & 1,562  & 1,551.0  & 1,683  & 3,722.6  & 3,125  & \textcolor[rgb]{ 1,  0,  0}{\textbf{3,067 }} & 3,428.2  & 3,433  & 4,052.8  & 5,046  & 5,070  & 4,154.4  & \textcolor[rgb]{ 1,  0,  0}{\textbf{3,965 }}  \\
    bitonic\_mesh & 649.2  & 585   & \textcolor[rgb]{ 1,  0,  0}{\textbf{582 }} & 694.4  & 613   & 1,212.6  & 1,027  & \textcolor[rgb]{ 1,  0,  0}{\textbf{982 }} & 1,164.8  & 1,098  & 1,523.8  & 1,479  & 1,493  & 1,504.0  & \textcolor[rgb]{ 1,  0,  0}{\textbf{1,332 }} \\
    cholesky\_bdti & 1,285.0  & \textcolor[rgb]{ 1,  0,  0}{\textbf{1,154 }} & \textcolor[rgb]{ 1,  0,  0}{\textbf{1,154 }} & 1,435.0  & 1,179  & 2,067.2  & 1,824  & \textcolor[rgb]{ 1,  0,  0}{\textbf{1,820 }} & 2,079.4  & 1,901  & 2,370.0  & \textcolor[rgb]{ 1,  0,  0}{\textbf{2,106 }} & \textcolor[rgb]{ 1,  0,  0}{\textbf{2,106 }} & 2,451.6  & 2,117  \\
    cholesky\_mc & 471.0  & \textcolor[rgb]{ 1,  0,  0}{\textbf{280 }} & 282   & 671.0  & 304   & 1,094.8  & 1,038  & \textcolor[rgb]{ 1,  0,  0}{\textbf{955 }} & 1,186.0  & 1,056  & \textcolor[rgb]{ 1,  0,  0}{\textbf{1,117.0 }} & 1,231  & 1,135  & 1,277.6  & 1,245  \\
    dart  & 857.6  & 850   & \textcolor[rgb]{ 1,  0,  0}{\textbf{803 }} & 828.0  & 806   & 1,402.4  & 1,337  & 1,377  & 1,425.8  & \textcolor[rgb]{ 1,  0,  0}{\textbf{1,292 }} & 1,715.6  & 1,753  & 2,045  & 1,746.8  & \textcolor[rgb]{ 1,  0,  0}{\textbf{1,652 }} \\
    denoise & 411.4  & 701   & \textcolor[rgb]{ 1,  0,  0}{\textbf{410 }} & 424.4  & 511   & 905.6  & 967   & 963   & 852.0  & \textcolor[rgb]{ 1,  0,  0}{\textbf{748 }} & 1,100.0  & 1,746  & 1,107  & 1,107.4  & \textcolor[rgb]{ 1,  0,  0}{\textbf{986 }} \\
    des90 & 471.2  & 379   & \textcolor[rgb]{ 1,  0,  0}{\textbf{374 }} & 617.4  & 407   & 726.6  & 675   & \textcolor[rgb]{ 1,  0,  0}{\textbf{621 }} & 825.8  & 677   & 1,047.6  & 1,017  & 920   & 1,064.8  & \textcolor[rgb]{ 1,  0,  0}{\textbf{877 }} \\
    directrf & 615.6  & 638   & \textcolor[rgb]{ 1,  0,  0}{\textbf{513 }} & 653.6  & 571   & \textcolor[rgb]{ 1,  0,  0}{\textbf{743.4 }} & 833   & 812   & 758.0  & 873   & 1,174.0  & 1,456  & 1,202  & 1,194.2  & \textcolor[rgb]{ 1,  0,  0}{\textbf{1,120 }} \\
    gsm\_switch & 1,613.0  & 4,435  & 1,829  & \textcolor[rgb]{ 1,  0,  0}{\textbf{1,612.0 }} & 1,834  & \textcolor[rgb]{ 1,  0,  0}{\textbf{3,304.4 }} & 6,541  & 4,825  & 3,393.6  & 3,704  & 3,898.2  & 13,484  & 5,073  & \textcolor[rgb]{ 1,  0,  0}{\textbf{3,819.4 }} & 4,298  \\
    LU\_Network & \textcolor[rgb]{ 1,  0,  0}{\textbf{523.0 }} & 525   & 524   & 523.4  & 531   & 786.8  & 904   & 884   & \textcolor[rgb]{ 1,  0,  0}{\textbf{785.8 }} & 791   & \textcolor[rgb]{ 1,  0,  0}{\textbf{1,380.0 }} & 1,432  & 1,424  & 1,403.8  & 1,528  \\
    LU230 & 4,472.8  & \textcolor[rgb]{ 1,  0,  0}{\textbf{3,309 }} & 3,363  & 4,692.0  & 3,618  & 6,587.4  & 5,143  & \textcolor[rgb]{ 1,  0,  0}{\textbf{4,758 }} & 6,526.6  & 5,324  & 7,750.8  & \textcolor[rgb]{ 1,  0,  0}{\textbf{5,944 }} & 6,701  & 7,930.2  & 6,221  \\
    mes\_noc & 694.8  & 648   & \textcolor[rgb]{ 1,  0,  0}{\textbf{632 }} & 821.2  & 661   & 1,533.6  & 1,228  & \textcolor[rgb]{ 1,  0,  0}{\textbf{1,184 }} & 1,425.2  & 1,375  & 1,725.6  & \textcolor[rgb]{ 1,  0,  0}{\textbf{1,462 }} & 1,466  & 1,676.6  & 1,495  \\
    minres & 308.4  & \textcolor[rgb]{ 1,  0,  0}{\textbf{206 }} & \textcolor[rgb]{ 1,  0,  0}{\textbf{206 }} & 329.4  & 242   & 436.8  & 416   & \textcolor[rgb]{ 1,  0,  0}{\textbf{416 }} & 488.0  & 520   & 716.4  & \textcolor[rgb]{ 1,  0,  0}{\textbf{582 }} & 582   & 705.2  & 691  \\
    neuron & 327.4  & 243   & \textcolor[rgb]{ 1,  0,  0}{\textbf{242 }} & 344.6  & 412   & 509.2  & 503   & \textcolor[rgb]{ 1,  0,  0}{\textbf{480 }} & 501.8  & 559   & 597.8  & \textcolor[rgb]{ 1,  0,  0}{\textbf{587 }} & 589   & 644.4  & 706  \\
    openCV & 1,271.6  & 449   & \textcolor[rgb]{ 1,  0,  0}{\textbf{434 }} & 936.6  & 1,438  & 1,410.8  & \textcolor[rgb]{ 1,  0,  0}{\textbf{897 }} & 911   & 1,480.8  & 1,637  & 1,961.4  & \textcolor[rgb]{ 1,  0,  0}{\textbf{1,286 }} & 1,281  & 1,878.0  & 2,097  \\
    segmentation & \textcolor[rgb]{ 1,  0,  0}{\textbf{106.0 }} & 118   & 119   & 128.0  & 139   & 450.6  & 513   & 451   & 466.6  & \textcolor[rgb]{ 1,  0,  0}{\textbf{439 }} & 508.8  & 515   & \textcolor[rgb]{ 1,  0,  0}{\textbf{508 }} & 513.2  & 548  \\
    SLAM\_spheric & \textcolor[rgb]{ 1,  0,  0}{\textbf{1,061.0 }} & \textcolor[rgb]{ 1,  0,  0}{\textbf{1,061 }} & \textcolor[rgb]{ 1,  0,  0}{\textbf{1,061 }} & \textcolor[rgb]{ 1,  0,  0}{\textbf{1,061.0 }} & \textcolor[rgb]{ 1,  0,  0}{\textbf{1,061 }} & 3,410.4  & 3,362  & \textcolor[rgb]{ 1,  0,  0}{\textbf{3,233 }} & 3,369.6  & 3,522  & 4,106.4  & 3,939  & \textcolor[rgb]{ 1,  0,  0}{\textbf{3,827 }} & 4,402.8  & 4,202  \\
    sparcT1\_chip2 & \textcolor[rgb]{ 1,  0,  0}{\textbf{866.6 }} & 1,226  & 870   & 912.4  & 904   & 1,361.6  & 1,451  & 1,675  & 1,479.4  & \textcolor[rgb]{ 1,  0,  0}{\textbf{1,332 }} & 1,671.6  & 1,896  & 1,608  & 1,656.0  & \textcolor[rgb]{ 1,  0,  0}{\textbf{1,588 }} \\
    sparcT1\_core & \textcolor[rgb]{ 1,  0,  0}{\textbf{974.0 }} & 1,027  & 977   & 974.0  & 1,014  & \textcolor[rgb]{ 1,  0,  0}{\textbf{1,829.0 }} & 2,693  & 2,127  & 1,844.8  & 1,859  & 2,297.8  & 3,176  & 3,094  & \textcolor[rgb]{ 1,  0,  0}{\textbf{2,272.4 }} & 2,363  \\
    sparcT2\_core & \textcolor[rgb]{ 1,  0,  0}{\textbf{1,181.0 }} & 1,272  & 1,188  & 1,181.0  & 1,288  & 2,440.4  & 3,423  & 2,427  & 2,415.2  & \textcolor[rgb]{ 1,  0,  0}{\textbf{2,284 }} & 2,987.0  & 4,262  & 3,973  & \textcolor[rgb]{ 1,  0,  0}{\textbf{2,969.4 }} & 3,048  \\
    stap\_qrd & 561.4  & 379   & 464   & 556.8  & \textcolor[rgb]{ 1,  0,  0}{\textbf{379 }} & 896.6  & 768   & \textcolor[rgb]{ 1,  0,  0}{\textbf{682 }} & 1,031.0  & 723   & 1,146.6  & 1,083  & 1,048  & 1,162.6  & \textcolor[rgb]{ 1,  0,  0}{\textbf{999 }} \\
    stereo\_vision & 227.8  & 178   & \textcolor[rgb]{ 1,  0,  0}{\textbf{169 }} & 170.8  & 188   & 385.6  & 390   & \textcolor[rgb]{ 1,  0,  0}{\textbf{347 }} & 380.0  & 391   & 510.8  & \textcolor[rgb]{ 1,  0,  0}{\textbf{460 }} & 511   & 464.2  & 475  \\
    all   & 21,430.6  & 21,175  & 17,758  & 21,118.0  & 19,783  & 37,218.4  & 39,058  & 34,997  & 37,308.4  & 35,538  & 45,360.0  & 55,942  & 46,763  & 45,999.0  & 43,553  \\
    mean  & 1.000  & 0.993  & 0.856  & 1.032  & 0.965  & 1.000  & 1.037  & 0.954  & 1.019  & 0.982  & 1.000  & 1.152  & 1.011  & 1.011  & 0.978   \\
    \hline
    \end{tabular}}
    \label{tab:Titancutk234}%
\end{table}%

Among the 22 test cases in this group, the proposed algorithm achieved the lowest total cut sizes for 2, 3, and 4-way partitioning, at 19{,}783; 35{,}538; and 43{,}553 respectively. These results surpass those of {KaHyPar}, {hMetis}, and {Mt-KaHyPar} across all $k$ values, and are also superior to {K-SpecPart} at $k=2$ and $k=4$, demonstrating its enhanced effectiveness in minimizing cut hyperedges. This advantage is further quantified by the average cut size ratios relative to KaHyPar: $0.965\times$ (2-way), $0.982\times$ (3-way), and $0.978\times$ (4-way), indicating an overall improvement of approximately 2\%--4\%. Significant performance gains were observed on instances such as \texttt{des90}, \texttt{LU230}, \texttt{mes\_noc}, and \texttt{stap\_qrd}. Furthermore, on certain instances, the proposed algorithm yields a 24\% improvement over K-SpecPart, 55\% over Mt-KaHyPar, and 72\% over hMetis.

\subsection{Effectiveness tests}
On both ISPD98 and Titan23 benchmarks, our algorithm incurs significantly higher computational costs than KaHyPar, with average runtimes being 1.8$\times$ to nearly 3$\times$ longer, depending on the value of $k$. While our method demonstrates superior partition quality, this result highlights a clear trade-off and indicates substantial room for improving its computational efficiency. The comparison was conducted under a "fair time budget" based on the virtual instance methodology in~\cite{sanders2023high} to ensure an equitable comparison with the stochastic KaHyPar.

For each test instance, we first execute our algorithm once, recording its runtime \( T_{\text{Our}} \) and solution quality \( Q_{\text{Our}} \). We then run KaHyPar repeatedly with different random seeds, accumulating execution times until the total time \( T_{\text{KaHyPar}} \) satisfies:
\[
\frac{\max(T_{\text{Our}}, T_{\text{KaHyPar}})}{\min(T_{\text{Our}}, T_{\text{KaHyPar}})} \leq \rho
\]
with threshold \( \rho = 1.2 \). This allows the faster KaHyPar to use additional time for multiple attempts. We compare the best solution quality \( Q_{\text{KaHyPar-best}} \) found by KaHyPar against \( Q_{\text{Our}} \).
When a single KaHyPar run exceeds \( T_{\text{Our}} \), we include it with a relaxed threshold \( \rho = 1.5 \), ensuring fair evaluation on challenging instances.

Figure \ref{sametime} shows the solution quality comparison between our algorithm and KaHyPar under equal time budgets for $k=2,3,4$ partitions on both ISPD98 and Titan23 benchmark sets. Among the six comparison groups, our algorithm excels notably in the ISPD98 dataset at $k=4$, successfully solving 61\% of the problems with the best values, surpassing KaHyPar's 38.8\%. In the ISPD98 dataset, when the performance ratio exceeds 1.05, our algorithm demonstrates a rapid improvement in solving capability, overtaking KaHyPar and approaching a near-optimal performance of 1. In contrast, for the Titan23 dataset at $k=2$ and $k=3$, the quality of our algorithm falls short of KaHyPar's, though at $k=4$, the performance of both algorithms becomes comparable.

\begin{figure}
  \centering
  \includegraphics[width=15cm, height=10cm]{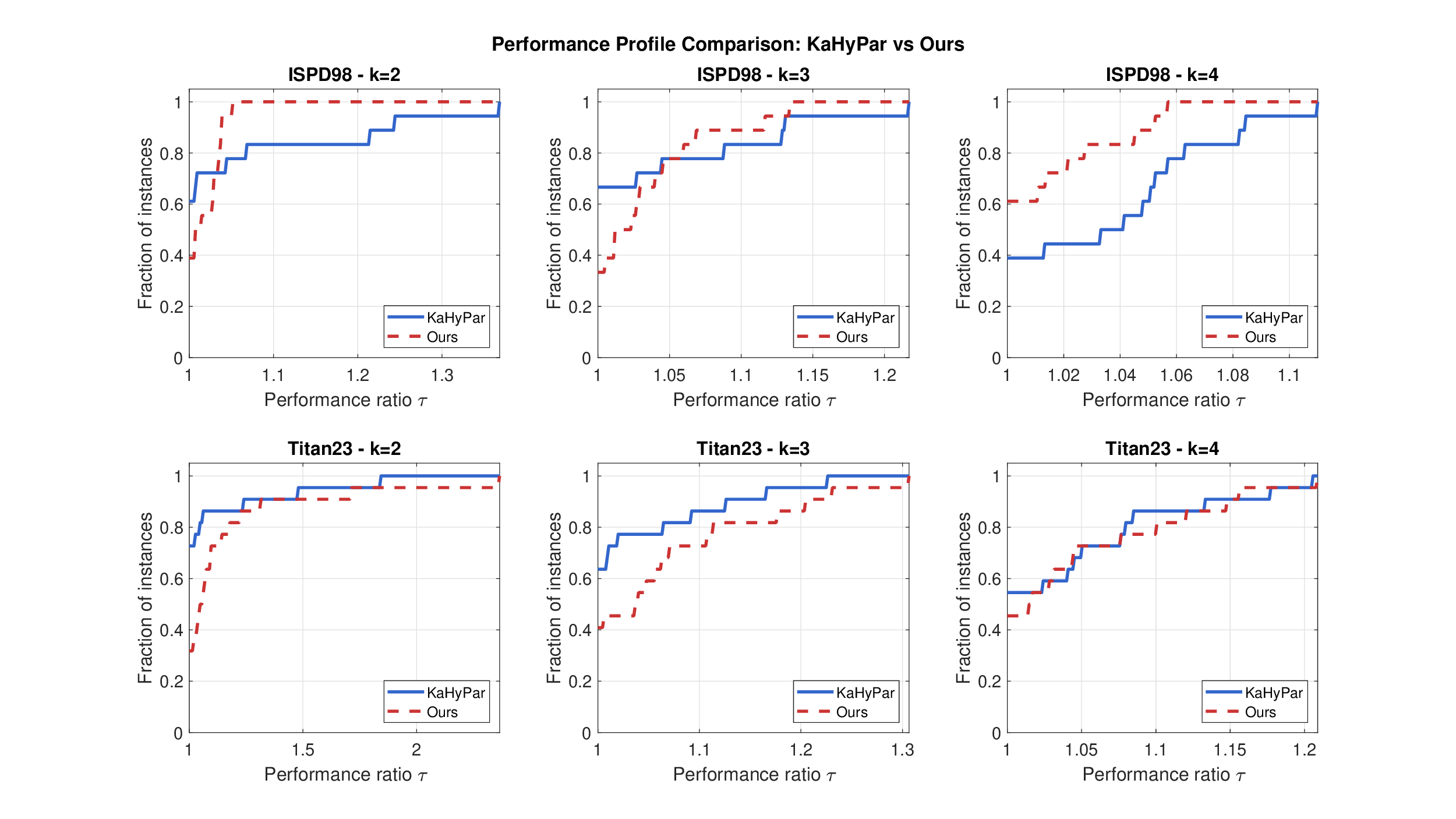}
  \caption{Solution quality comparison against KaHyPar under equal time budgets on ISPD98 and Titan23 benchmarks for $k=2,3,4$.}
  \label{sametime}
\end{figure}


The Average Performance Ratio (APR) evaluates an algorithm's overall quality by averaging its relative performance across all instances. It is computed as:
\[
\text{APR} = \frac{1}{n} \sum_{i=1}^{n} \frac{P_i}{P_i^*},
\]
where $P_i$ is the algorithm's performance on instance $i$, $P_i^*$ is the best performance achieved by any compared algorithm on that instance, and $n$ is the total number of instances. An APR value of 1 indicates optimal performance. 

The comparison between our method and KaHyPar is presented in Figure~\ref{wins}, which shows the performance comparison between the two algorithms on the ISPD98 and Titan23 datasets, revealing distinct performance trends. On the ISPD98 dataset, the performance of our algorithm (Ours) consistently improves as the number of partitions $k$ increases, demonstrating a clear advantage at $k=4$. Both the average performance ratio (1.013) and the number of better solutions obtained (11) outperform KaHyPar (1.035 and 7). In contrast, on the Titan23 dataset, KaHyPar dominates at $k=2$ and $k=3$. However, as $k$ increases to 4, the gap between the two algorithms narrows significantly, and their performance becomes comparable.

\begin{figure}
  \centering
  \includegraphics[width=14cm, height=8cm]{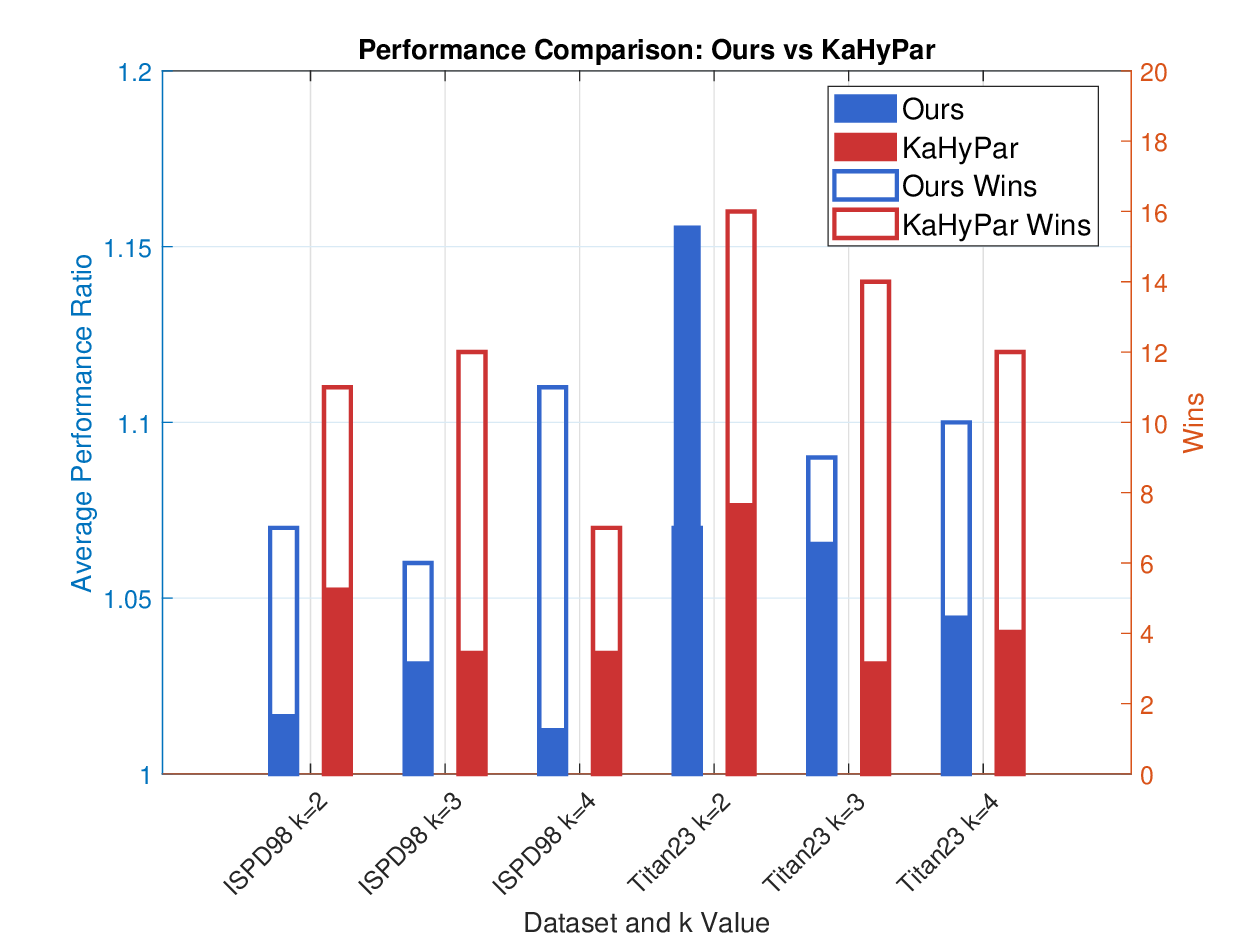}
  \caption{Solution quality comparison against KaHyPar under equal time budgets on ISPD98 and Titan23 benchmarks for $k=2,3,4$.}
  \label{wins}
\end{figure}


\subsection{Visual analysis of hMetis refinement results}

To validate the effectiveness of the improvement algorithm proposed in Section~\ref{sect5}, we applied our method to enhance the partitioning results of hMetis. Since KaHyPar reports the average result over five runs with different random seeds and does not provide a corresponding specific partition, we do not perform improvement analysis on KaHyPar.

The improvement results are shown in the boxplot in Figure~\ref{box}, which compares the cut sizes before and after refinement for different partition counts (2-way, 3-way, and 4-way). In the boxplot, the bold black line represents the median, while the top and bottom edges of the box correspond to the 75th percentile (Q3) and 25th percentile (Q1), covering the middle 50\% of the data. The whiskers show the range of non-outlier values, defined by \([Q1 - 1.5 \times \mathrm{IQR}, Q3 + 1.5 \times \mathrm{IQR}]\), where \(\mathrm{IQR} = Q3 - Q1\). Data points outside this range are considered outliers. The red dot above the box indicates the mean, showing its difference from the median.

\begin{figure}
  \centering
  \includegraphics[width=14cm, height=8cm]{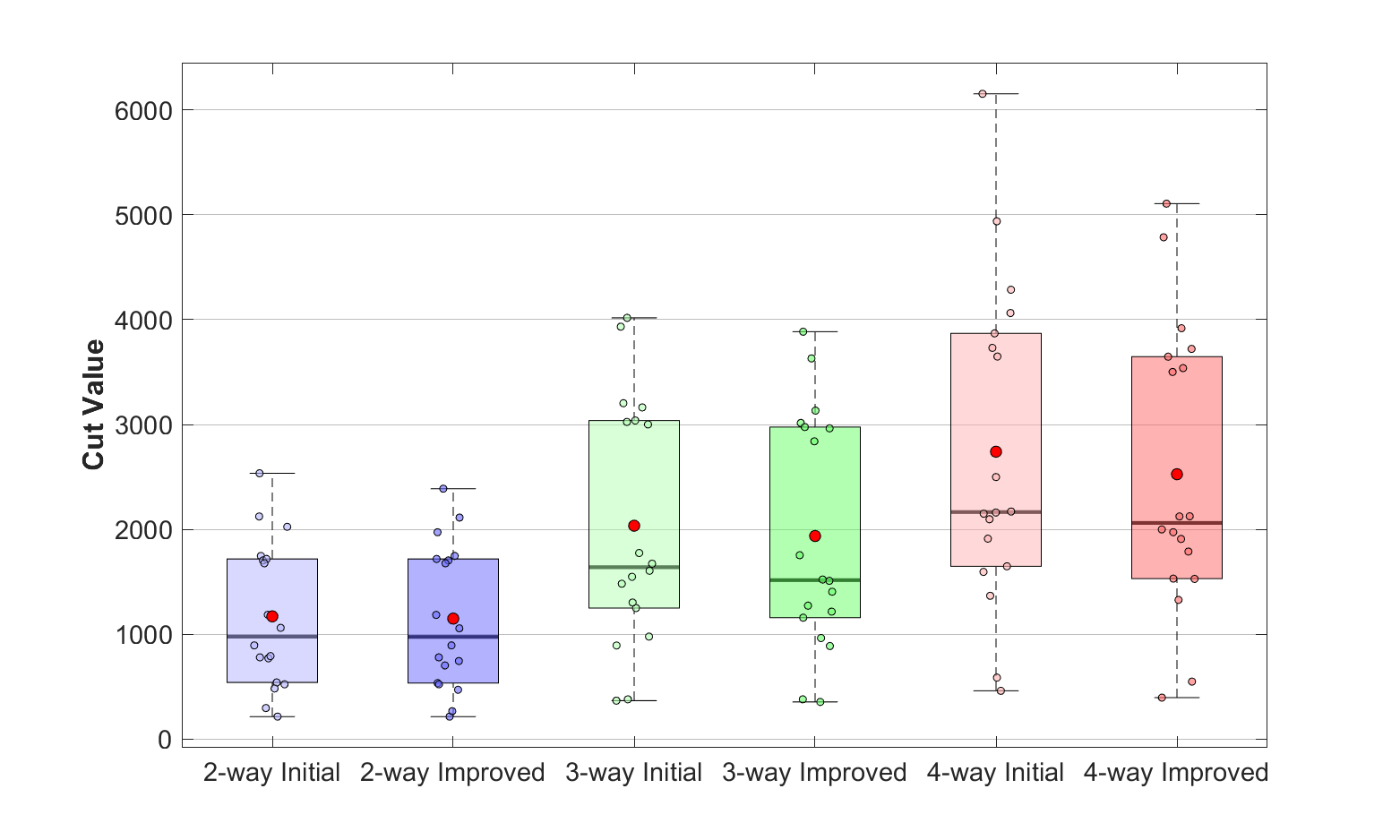}
  \caption{Comparison of initial and improved partitioning results for hMetis.}
  \label{box}
\end{figure}

As shown in the figure, the refined partitions exhibit lower cut sizes in most instances, with both the median and mean values outperforming the initial results, indicating a significant improvement in overall partitioning quality. Specifically, for 2-way, 3-way, and 4-way partitioning, the average cut sizes were reduced to 97.9\%, 95.3\%, and 92.5\% of the original values, respectively.

Moreover, the distributions of the refined results are more concentrated, with fewer outliers, reflecting improved algorithmic stability. Among all 54 partitioning results, 83\% showed improvements, with the maximum reduction in cut size reaching up to 16\% for a single instance. These findings further confirm the effectiveness and robustness of the proposed refinement strategy in improving partitioning quality.

\subsection{Summary of parameter sensitivity analysis}

To further examine the influence of key parameter settings on partitioning quality and runtime, we conduct three groups of sensitivity experiments on the ISPD98 benchmark under the $k=2$ partitioning setting. The results are presented in Figure~\ref{Parameteranalysis}, where the blue line denotes the total cut size and the red line denotes the total runtime. Each experiment adopts a single-variable control strategy to ensure fair comparison.


\textbf{(a) Clustering number $p$:} We randomly select two values of \(p\) within the range \([2k, n/2]\) and conduct tests on five groups. The first group in the figure uses the parameters from the literature.

\textbf{(b) Number of initial partitionings $num\_Init$:} Increasing the number of initial partitioning candidates from 2 to 20 slightly improves cut size but leads to linearly increasing runtime, showing a trade-off between quality and efficiency.

\textbf{(c) Weighting coefficient control:} To validate the rationality of the default values of parameters $\lambda_1$, $\lambda_2$, $\xi_1$, and $\xi_2$, we conducted systematic perturbation experiments. The experiments followed a controlled-variable approach, where only one parameter was altered at a time while the others remained fixed at their default values. Specifically, for each parameter, a number of configurations equivalent to the default parameter set in this study were randomly sampled from a uniform distribution over [0,1] and evaluated under the same experimental conditions.

The results demonstrate that the default parameter setting achieves an optimal balance between cut size and runtime. Any deviation from this configuration leads to a degradation in either solution quality or computational efficiency, with variations in $\xi_2$ having the most pronounced impact on algorithm performance.

From our findings, we observe that the default
setting of the hyperparameters is a reasonable choice.

\begin{figure*}
  \centering
  \includegraphics[width=18cm, height=4cm]{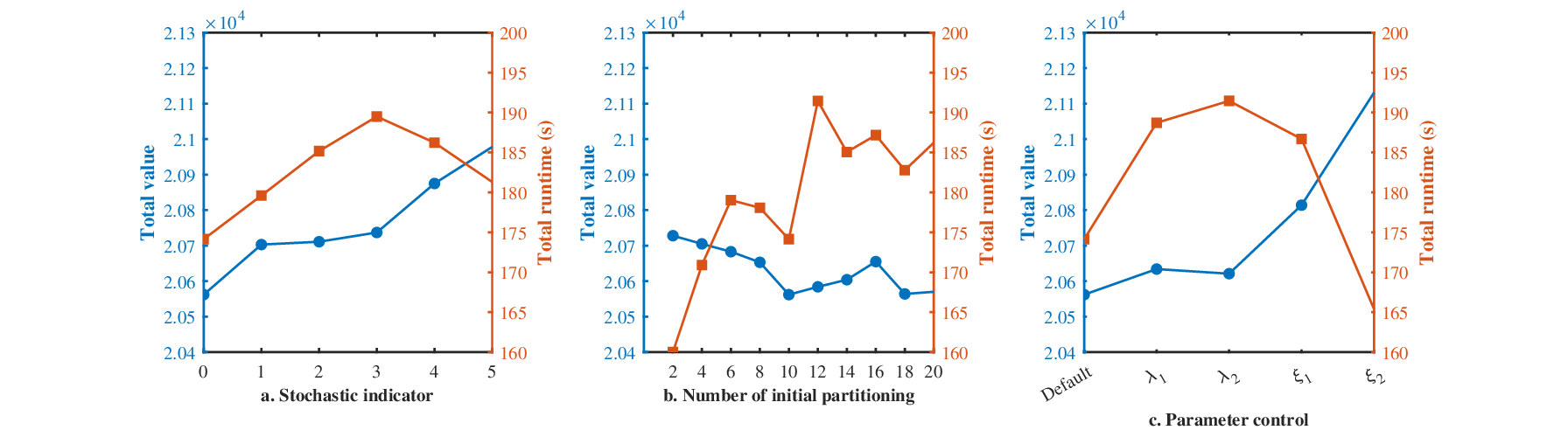}
  \caption{Sensitivity analysis of key parameters on partitioning quality and runtime.}
  \label{Parameteranalysis}
\end{figure*}

\section{Conclusion}\label{sect7}

This paper presents a novel multi-way graph partitioning method that integrates multi-objective optimization, MST construction, and a clustering strategy. The proposed two-stage framework effectively balances computational efficiency and partition quality across different data scales. Experimental results demonstrate that the framework achieves superior partitioning quality compared to state-of-the-art algorithms, exhibiting significant advantages particularly on weighted vertex sets. Furthermore, the proposed refinement strategy demonstrates strong generality, capable of effectively enhancing the results of existing partitioners. A comprehensive evaluation confirms that this work provides a high-quality and efficient solution for hypergraph partitioning, showcasing strong competitiveness and practical value.

\section*{Acknowledgments}
This research was supported by National Natural Science Foundation of China (No. 12261019), Natural Science Basic Research Program Project of Shaanxi Province (No. 2024JC-YBMS-019) and the Fundamental Research Funds for the Central Universities and the Innovation Fund of Xidian University (No. YJSJ25009).






\bibliographystyle{unsrt}        
\bibliography{thebibliography}


\end{document}